\documentclass[review]{elsarticle}

\usepackage{lineno,hyperref}
\modulolinenumbers[5]

\journal{Journal of \LaTeX\ Templates}

\usepackage{times}
\usepackage{latexsym} 
\usepackage{multirow}

\usepackage{color,soul}

\usepackage{todonotes}

\usepackage{amsmath}
\usepackage{amssymb}
\usepackage{mathrsfs}
\usepackage{xcolor}
\usepackage{soul}
\usepackage{amsfonts}
\usepackage{graphicx}
\usepackage{framed}
\usepackage{tikz}
\usetikzlibrary{arrows,shapes.geometric,shapes.arrows}
\usepackage{amsthm} 
\usepackage[ruled,vlined,linesnumbered]{algorithm2e} 
\usepackage{makecell}
\usepackage{tikz}

\newtheorem{theorem}{Theorem}

\newtheorem{lemma}{Lemma}

\newtheorem{definition}{Definition}
\newtheorem{example}{Example}

\SetKwFunction{NetworkTight}{NetworkTight}
\SetKwFunction{EstimateError}{EstimateError}
\SetKwFunction{Solve}{Solve}

\DeclareMathOperator{\support}{support}
\DeclareMathOperator{\Trim}{Trim}

\DeclareMathOperator{\Sequence}{Sequence} 
\DeclareMathOperator{\Convolve}{Convolve}

\DeclareMathOperator{\ch}{\operatorname{children}}

\SetKwFunction{Convolv}{Conv}
\SetKwFunction{LTrim}{LTrim}
\SetKwFunction{Network}{Network}
\SetKwFunction{LNetwork}{LNetwork}
\SetKwFunction{Parallel}{Parallel}
\SetKwFunction{GenNetwork}{GenNetwork}
\SetKwFunction{Solve}{Solve}
\SetKwFunction{NetworkTight}{NetworkTight}
\SetKwFunction{EstimateError}{EstimateError}

\usepackage{array}
\usepackage{ragged2e}
\newcolumntype{X}[1]{>{\RaggedRight\hspace{0pt}}p{#1}}

\begin{document}

\begin{frontmatter}

\title{Estimating the Probability of Meeting a \\ Deadline in Hierarchical Plans\footnote{A short version of this paper was presented in IJCAI 2015~\cite{cohen2015estimating}}.}

\author{Liat Cohen}
\author{Solomon Eyal Shimony}
\author{Gera Weiss}
\address{Computer Science Department\\
Ben Gurion University of The Negev\\
Beer-Sheva, Israel 84105\\}

\begin{abstract}
	Given a hierarchical plan (or schedule) with uncertain task times, we propose a deterministic polynomial (time and memory) algorithm for estimating the probability that its meets a deadline, or, alternately, that its {\em makespan} is less than a given duration. Approximation is needed as it is known that this problem is NP-hard even for sequential plans (just, a sum of random variables). In addition, we show two new complexity results: (1) Counting the number of events that do not cross deadline is \#P-hard; (2)~Computing the expected makespan of a hierarchical plan is NP-hard. For the proposed approximation algorithm, we establish formal approximation bounds and show that the time and memory complexities grow polynomially with the required accuracy, the number of nodes in the plan, and with the size of the support of the random variables that represent the durations of the primitive tasks. We examine these approximation bounds empirically and demonstrate, using task networks taken from the literature, how our scheme outperforms sampling techniques and exact computation in terms of accuracy and run-time. As the empirical data shows much better error bounds than guaranteed, we also suggest a method for tightening the bounds in some cases. 
\end{abstract}  


\end{frontmatter}


\section{Introduction}

Numerous planning tools produce plans that call for executing tasks non-linearly.
Usually, such plans are represented as a tree, where the
leaves indicate primitive tasks, and other nodes represent compound tasks consisting of executing their sub-tasks either in parallel (also called ``concurrent'' tasks~\cite{gabaldon2002programming}) or in sequence ~\cite{erol1994htn,nau1998control,nau2003shop2,kelly2008offline}.

Given such a hierarchical plan representation, it is frequently of interest to evaluate its desirability in terms of resource consumption, such as fuel, cost, or time. The answer to such questions can be used to decide which of a set of plans, all valid as far as achieving the goal(s) are concerned, is better given a user-specified utility function. Another reason to compute these distributions is to support runtime monitoring of resources, generating alerts to the execution software or human operator if resource consumption in practice has a high probability of surpassing a given threshold.

While most tools aim at good average performance of the plan, in which case one may ignore the 
full distribution and consider only the expected resource consumption~\cite{bonfietti2014disregarding}, our paper focuses
on providing guarantees for the probability of meeting deadlines. This type of analysis is needed, e.g., in
Service-Level-Agreements (SLA) where guarantees of the form: ``response time less than 1mSec in at least 95\% of the cases'' are common~\cite{buyya2011sla}, 
Section~\ref{sec:discussion} discusses additional related work.
We assume that a hierarchical plan is given in the form of a tree, with uncertain resource consumption of the primitive actions in the network, provided as a probability distribution. The problem is to compute a property of interest of the distribution for the entire task network. In this paper, we focus mainly on the issue of computing the probability $P(X_\tau \leq T)$ of satisfying a deadline $T$ where $X_\tau$ is a random variable describing the distribution of the makespan of the plan. Since in the above-mentioned applications for these computations, one needs results in real-time (for monitoring) or multiple such computations (in comparing candidate plans), efficient computation here is crucial, and is more important than in, e.g., off-line planning.

The decision problem we analyze is: given a task tree and a deadline, does the probability of meeting this deadline is above a given threshold. We show that this decision problem is NP-hard (see Section~\ref{sec:complexity}), the first contribution of this paper. We propose a deterministic polynomial-time approximation scheme for this problem, a second contribution of this paper. Error bounds are analyzed and are shown to be tight. For discrete random variables with finite support, finding the distribution of the maximum can be done in low-order polynomial time. However, when compounded with errors generated due to approximation in subtrees, handling this case requires careful analysis of the resulting error. The approximations developed for both sequence and parallel nodes are combined into an overall algorithm for task trees, with an analysis of the resulting error bounds, yielding a polynomial-time (additive error) approximation scheme for computing the probability of satisfying a deadline for the complete network, another contribution of this paper. 
We also consider computing {\em expected} makespan. Since for discrete random variables, in parallel nodes one can compute an exact distribution efficiently, it is easy to compute an {\em expected makespan} in this case as well as for sequence nodes. Despite that,  we show that for trees with both parallel and sequence nodes, computing the expected makespan is hard.


Experiments are provided in order to examine the quality of approximation in practice when compared to the theoretical error bounds. A simple sampling scheme is also provided as a yardstick, even though the sampling does not come with error guarantees, but only bounds in probability. Finally, we examine our results in light of related work in the fields of planning and scheduling, as well as
probabilistic reasoning. 

\section{Problem statement}\label{sec:formal}

We are given a hierarchical plan represented as a task tree consisting of three types of nodes: primitive actions as leaves, sequence nodes, and parallel nodes.
Primitive action nodes contain distributions over their resource consumption. Many resources can be modeled using the proposed approach. 
For example if the resource of interest is memory, tasks running in parallel use the sum of the memory space of each of the tasks; if they run in sequence, only the maximum thereof is needed. Note that the role parallel and sequence node is inversed here. We will assume henceforth, in order to be more concrete, that the resource of interest is time, i.e., that parallel nodes represent maximum and sequence nodes represent sum.


The tasks trees that we analyze are composed of sequence nodes, parallel nodes and leaves that we call primitive tasks. A sequence node represents a composition of tasks in sequence. Its makespan is the the sum of the makespans of its child nodes. A parallel node represents a composition of tasks in parallel. Its makespan is the maximum of the makespan of its child nodes.
The makespans of the primitive nodes at the leafs of the tree are uncertain, and described as probability distributions.
We assume  that the distributions are  independent (but {\em not} necessarily identically distributed).
We also assume initially that the random variables are discrete and have finite support (i.e. the number of values for which the probability is non-zero is finite).
In this paper, we use the term $P_X$ to denote the (discrete) probability distribution of a (discrete) random variable $X$, i.e., the list of the probabilities of the outcomes, also known as the probability mass function. We use the term $F_X$ to denote the cumulative distribution function, i.e., $F_X(x)$ is the probability that $X$ will take a value less than or equal to $x$. More specifically, as the resource of interest is completion time, we associate each leaf node, $v$, with the random variable, $X_v$, that represents the distribution of the 
a completion-time distribution $P_{X_v}$ and a cumulative distribution function (CDF) $F_{X_v}$.

For clarity, we use the following standard for notations in the paper. First, we denote all random variables by symbols of the form $X_{name}$ where in the subscript we put the name of the variable. 
Second, we use the symbol $\tau_{name}$ where the subscript contains a name of a node to denote the subtree starting at this node, i.e., the node is the root of the subtree. Note that $X_{\tau_{v}}$ and $X_v$ have the same meaning - a random variable representing the makespan of the subtree $\tau_{v}$.
Last, we use primed versions of random variables to denote approximations, i.e., the symbol $X'_{name}$ denotes an approximation of the random variable $X_{name}$.

The main computational problem analyzed in this paper is the {\em deadline problem}: 
\begin{definition}\label{Def:Deadline}
	Given a task tree $\tau $ and a deadline $T$, the \emph{deadline problem} is to compute $F_{X_\tau}(T)$.
\end{definition}

In words, given a task tree $\tau $ and a deadline $T$, we ask what is the probability that the plan modeled by $\tau$ terminates in time less than $T$?

The above {\em deadline problem} reflects a step utility function: a 
constant positive utility $U$ for all $t$ less than or equal to a deadline 
time $T$, and $0$ for all $t>T$. We also briefly consider a linear utility function, requiring
computation of the expected completion time of $\tau$, and show that this {\em expectation problem} 
is also NP-hard. 
The duration of a sequence node $s$ is a random variable:

$$
X_s = \sum_{c\in \ch(s)} X_c
$$
where $X_c$ is a random variable representing the duration of the plan modeled by the subtree rooted at the child node $c$ and $\ch(s)$ is the set of children of $s$.


Likewise, the duration distribution of a parallel node $p$ is a random variable:
$$
X_p = \max_{c\in \ch(p)} X_c
$$
Let $r$ be the root node of a task tree $\tau$ and $X_r$ be a random variable representing the duration distribution of the root.
Then the probability that the plan meets the deadline $T$ is $F_{X_r}(T)$.
Thus we need to compute the CDF, which is NP-hard~\cite{mohring2001scheduling}. We show how to deterministically approximate the CDF of the root with additive error at most $\varepsilon$ in time polynomial in $1/\varepsilon$.

Figure~\ref{fig:hierarchical plan} is a simple hierarchical plan example. The set of nodes $V$ represented by $\lbrace A,B,C,a,b,c,d,e \rbrace $ and the type of each task node implied by its shape, $A$ and $C$ are sequence nodes, $ B $ is a parallel node, and $ a, b, c, d, e $ are all primitive nodes. Every primitive node is associated with probability mass function (PMF) describes the completion-time distribution. In this case $P_{X_a}(x)=P_{X_b}(x)=P_{X_c}(x)=P_{X_d}(x)=P_{X_e}(x)$.
$$P_{X_a}(x)=
\begin{cases}
1/4 & \text{if } x=1; \\
3/4 & \text{if } x=4; \\
0 & otherwise.
\end{cases}$$
This tree gives execution instructions: run $a$ and $b$ in parallel, then run $c$ and $d$ in sequence and, when they finish, run $e$.

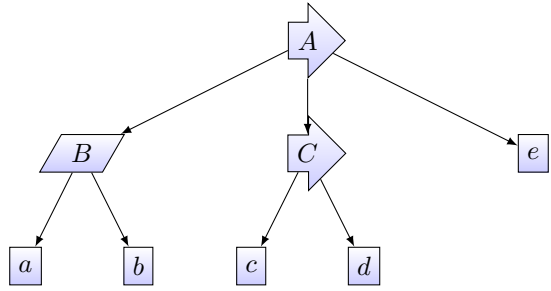
\begin{figure}
	\centering
	\begin{tikzpicture}[
	level/.style={sibling distance=30mm/#1},
	edge from parent/.style = {draw, -latex},
	sloped,
	every node/.style = {trapezium, trapezium left angle=60, trapezium right angle=-60, draw, align=center, top color=white, bottom color=blue!20,minimum height=0.5cm}
	]
	\node [single arrow] {$A$}  
	child { node {$B$} 
		child{ node [rectangle] {$a$}}
		child{ node [rectangle] {$b$}}
	}    
	child  { node [single arrow] {$C$} 
		child {node [rectangle] {$c$}}
		child {node [rectangle] {$d$}}
	}
	child {node [rectangle] {$e$}}
	;
	
	\end{tikzpicture}
	\caption{Example of a hierarchical plan and its graphical representation.}
	\label{fig:hierarchical plan}
\end{figure}

\begin{example}

Given a task tree $\tau$ as in figure ~\ref{fig:hierarchical plan}, we will compute the duration distribution in the form of PMF for each compound task node including the root node, and then we will return the probability for $\tau$ to satisfy a given deadline $T$. 
First, compute duration distribution for the node $B$. The children task nodes $a,b$, are executed in parallel, therefore we need to find the distribution over the maximum of nodes $a$ and $b$:  
$$P_{X_B}(x)=P_{\max(X_a,X_b)}(x)=
\begin{cases}
1/16 & \text{if } x=1; \\
15/16 & \text{if } x=4; \\
0 & otherwise.
\end{cases}$$

Second, compute duration distribution for node $C$. The children task nodes $c,d$, are executed in sequence, therefore, we need to use convolution in order to compute the sum of duration distributions of nodes $c$ and $d$:  
$$P_{X_C}(x)=P_{X_c+X_d}(x)=
\begin{cases}
1/16 & \text{if } x=2; \\
3/8 & \text{if } x=5; \\
9/16 & \text{if } x=8; \\
0 & otherwise.
\end{cases}$$

Last, compute duration distribution for node $A$, the root node. The children task nodes $B,C,e$ are executed in sequence, therefore, we need to use convolution in order to compute the sum of duration distributions of nodes $B$, $C$ and $e$: 
$$P_{X_{\tau}}=P_{X_A}(x)=P_{X_B+X_C+X_e}(x)=
\begin{cases}
1/1024 & \text{if } x=4; \\
3/128 & \text{if } x=7; \\
81/512 & \text{if } x=10; \\
27/64 & \text{if } x=13; \\
405/1024 & \text{if } x=16; \\
0 & otherwise.
\end{cases}$$

For every deadline $T$ we can easily return the probability for $\tau$ to satisfy $T$. If $T=8$, the probability is  $F_{X_{\tau}}(8)=25/1024$ . 
\end{example}
 
\section{Sum (sequence) nodes}\label{sec:seq}

The size of the support (number of non-zero probability values) of the sum of random variables
may be exponential in the number of variables, even for 2-valued variables.
In fact, as shown in~\cite{mohring2001scheduling}, computing the CDF of a sum of random variables at a given point is NP-hard.
We thus define a notion of approximation, 
which we call a {\bf Kolmogorov upper bound} (defined below), and supply an operator, we call $\Trim$, that 
produces such an approximation.

Let $X$ and $X'$ be random variables; 
the Kolmogorov distance~\cite{lilliefors1967kolmogorov} between $X$ and $X'$ is defined as:
\[
d_K(X,X') = \sup_{x} |F_{X'}(x)-F_X(x)|
\]
Our notion of approximation uses the Kolmogorov distance. Let $0\leq\varepsilon<1$. 
If the following equation holds:
\[
\forall ~ x, ~ F_X(x)+\varepsilon \geq F_{X'}(x) \geq F_X(x)
\]
we say that the random variable $X'$ is a Kolmogorov $\varepsilon$ upper bound approximation of $X$,
which we denote by $X' \succeq_\varepsilon X$. Contrapositively, we call $X$ a  Kolmogorov $\varepsilon$ lower
bound approximation of $X'$. Note that $X' \succeq_\varepsilon X$ implies that $d_K(X,X') \leq \varepsilon$,
but not vice-versa.

In our algorithms and examples,
the PMF of a random variable $X$ is represented by a list $D_X$, 
which consists of $d=(p, x)$ pairs, where $x\in \support(X)$
and $p$ is the probability $P(X=x)$. In the pair $d=(p, x)$, we denote the value $x$
by $val(d)$, and the probability $p$ by $prob(d)$.
For example, let $X$ be a random variable distributed as: $[0.1:1, 0.1:2, 0.8:4]$, and $X'$ a random variable distributed as $[0.2:1, 0.8:4]$. Then we have $X' \succeq_{0.1} X$. In order to achieve a Kolmogorov upper bound of $\varepsilon$,
the $\Trim_U$ operator removes consecutive domain values whose accumulated probability is
less than $\varepsilon$ and adds 
their probability mass to the element in the support that precedes them.

If the input $D_X$ to $\Trim$ is sorted in increasing order of $x$ (we denote this operator by $\Trim_U(X, \varepsilon)$)
then resulting variable $X'$ (represented by the output $D_{X'}$) 
is a Kolmogorov $\varepsilon$ upper bound of $X$.
Likewise, if $D_X$ sorted in decreasing order of $x$ 
(in this case we denote the $\Trim$ operator
by $\Trim_L(X, \varepsilon)$) then $X'$ is a Kolmogorov $\varepsilon$ 
lower bound of $X$. 

From now on, in order to simplify, we will use the notation $\Trim$ instead of $\Trim_U$. Note that we could have chosen to use $\Trim_L$ instead and get a symmetric version of all the results.



\begin{algorithm}
  \DontPrintSemicolon
  \SetKwFunction{Trimalgo}{Trim}
  \SetKwProg{myproc}{Procedure}{}{}
  { 
             $D_{X'}=()$, $p=0$ \;	  
  		$d_{prev} = first(D_X)$\;
                $tail = rest(D_X)$ \;
		\While{not-empty($tail$)}{
                        $d = first(tail)$ \;
			\eIf{$p+prob(d) \leq \varepsilon$}{
				$p=p+prob(d)$\;
			}
			{ Append $(val(d_{prev}), prob(d_{prev})+p)$ to $D_{X'}$\;
			  $d_{prev} = d$, $p=0$\;
			}
                        $tail=rest(tail)$
                 }
       		 Append $(val(d_{prev}), prob(d_{prev})+p)$ to  $D_{X'}$\;
  \Return $D_{X'}$}
     
\caption{$\Trim$($D_X$,$\varepsilon$)  }
\label{alg:Trim}
\end{algorithm}

Trimming decreases the support size, while introducing an error. The trick is
to keep the support size under control, while making sure that the error does
not increase beyond a desired tolerance.
Note that the size of the support can also be decreased by simple ``binning'' 
schemes, but these may not provide the desired guarantees.

%

\noindent We now show that with $D_X$ sorted in a increasing order,
$\Trim(D_X,\varepsilon)$ is an Kolmogorov $\varepsilon$ upper bound of $X$.
\begin{lemma} \label{Trim}
$\Trim(X,\varepsilon)\succeq_\varepsilon X$
\end{lemma}

\begin{proof}
Let $X'=\Trim(X,\varepsilon)$. 
Let $x_1{<}\cdots{<}x_m$ be the support of $X'$. 
Because $\Trim$ adds the probabilities of elements that were removed from the support of $X$ 
to the support element of $X'$ that precedes them, 
we have for all $i$:
\begin{equation}
P_{X'}(x_i) = P_X(x_i) +  P(x_i < X < x_{i+1})
\label{collect}
\end{equation}
(assuming $x_{m+1} = \infty$ for convenience.)
The value $P(x_i < X < x_{i+1})$ equals the value of $p$ in
the algorithm when the Append is performed, and  the loop invariant $0 \leq p \leq \varepsilon$ holds by construction.
For any value $x$, let $l_x=\max\{ i \colon x_i\leq x\}$, thus:
\[
F_{X'}(x) = \sum_{i=0}^{l_x} P_{X'}(x_i)
\]

From Equation~\eqref{collect} we get:
{
\begin{align*}
&F_{X'}(x) - F_{X}(x)  \\
 &{=} \sum_{i=0}^{l_x-1} (P_{X'}(x_i)- P(x_i\leq X <x_{i+1})) +(P_{X'}(x_{l_x})-P(x_{l_x} \leq X \leq x)) \\
&{=}P_{X'}(x_{l_x}) {-} (P(x_{l_x} {\leq} X {<} x_{l_x+1}) {-} P(x {<} X {<} x_{l_x+1})) \\
&{=}P(x < X < x_{l_x+1}) \in (0,\varepsilon]
\end{align*}}
\end{proof}
Showing that $X \succeq_\varepsilon \Trim_L(X,\varepsilon)$, i.e. that using inversely sorted $D_X$ results in a Kolmogorov $\varepsilon$ lower bound, is immediate due to symmetry.

To bound the amount of memory needed for our approximation algorithm,
the next lemma bounds the size of the support of the trimmed random variable:


\begin{lemma} \label{SizeD}
$|\operatorname{support}(\Trim(D_X,\varepsilon))| \leq 1/\varepsilon +1$
\end{lemma}

\begin{proof}
In the Trim operator, each ``Append'' adds 1 to the support of $X'$, and these occur only once inside
the ``else'' statement, and once outside the loop.
The ``else'' part of the loop occurs only if $p+prob(d)> \varepsilon$,
after which none of these elements of the list are reused in the ``if'' statement.
Therefore, as the sum of probabilities is 1, then the number of times
the ``else'' part is executed is at most ${1}/{\varepsilon}$. Thus the total support
is at most ${1}/{\varepsilon}+1$.
\end{proof}



The makespan of a $\Sequence$ operator is a random variable $X$, the sum 
of the random variables $X_i$ of its children.
Let $Y_i = X_i + Y_{i-1}$ for $1< i \leq n$, and $Y_1=X_1$.
Thus $X$ is distributed as $Y_n$.

As the children of a sequence node may be internal nodes in the task tree,
the input distributions may already be approximations. To keep the size of
the support small, we apply Trim after the addition of each random variable, i.e., 
we compute the random variables 
$Y'_i = \Trim(X'_i + Y'_{i-1}, \varepsilon)$,
where the $X'_i$ is a Kolmogorov $\varepsilon_i $ upper bound of $X_i$
and show that $Y'_n$ is a Kolmogorov $\delta $ upper bound of $X$, for an appropriate $\delta $.

The distribution of the sum of random variables $X'_i + Y'_{i-1}$ is 
computed by a discrete convolution, and Trim is
computed as in Algorithm~\ref{alg:Trim}. That is, our approximation for a $\Sequence$ operator (for $n\geq 2$) is given by: 
\begin{equation}
\label{eq:sequence}
\Sequence(X'_1,\dots , X'_n,  \varepsilon) 
= \Trim(\Convolve(X'_1, \Sequence(X'_2, \dots , X'_n, \varepsilon)), \varepsilon)
\end{equation}
We begin by bounding the approximation error propagated by convolution (sum of random variables ):

\begin{lemma} \label{Convolv}
For discrete random variables  $X_1,X_1',X_2,X_2'$ and $\varepsilon_1,\varepsilon_2 \in [0,1]$, 
if $X_1'\succeq_{\varepsilon_1} X_1$ and  $X_2'\succeq_{\varepsilon_2} X_2$,
then $X_1'+X_2'\succeq_{\varepsilon_1+\varepsilon_2} X_1+X_2$.
\end{lemma}

\begin{proof}
Define error functions $E_1$ and $E_2$ such that 
$F_{X_i'}(x) = F_{X_i}(x)+ E_i(x)$. By construction, we have $ 0 \leq E_i(x) \leq \varepsilon_i$.

By definition of sums of random variables and convolution, we have:
\begin{align*}
& F_{X_1'+X_2'}(y) = \sum_{x=-\infty}^{\infty}F_{X_1'}(y-x)P_{X_2'}(x)\\
&=\sum_{x=-\infty}^{\infty}(F_{X_1}(y-x)+E_1(y-x))P_{X_2'}(x)\\
&= \sum_{x=-\infty}^{\infty}F_{X_1}(y-x)P_{X_2'}(x)+\sum_{x=-\infty}^{\infty}E_1(y-x)P_{X_2'}(x)\\
&\leq\sum_{x=-\infty}^{\infty}F_{X_1}(y-x)P_{X_2'}(x)+\varepsilon_1\sum_{x=-\infty}^{\infty}P_{X_2'}(x)\\
&=\underbrace{\sum_{x=-\infty}^{\infty}F_{X_2'}(y-x)P_{X_1}(x)}_{F_{X_1+X_2'}(x)=F_{X_2'+X_1}(x)}+\varepsilon_1\sum_{x=-\infty}^{\infty}P_{X_2'}(x)\\
&=\sum_{x=-\infty}^{\infty}(F_{X_2}(y-x)+E_2(y-x))P_{X_1}(x) + \varepsilon_1\\
&=\sum_{x=-\infty}^{\infty}F_{X_2}(y-x)P_{X_1}(x)+\sum_{x=-\infty}^{\infty}E_2(y-x)P_{X_1}(x) + \varepsilon_1\\
&\leq\sum_{x=-\infty}^{\infty}F_{X_2}(y-x)P_{X_1}(x)+\varepsilon_2\sum_{x=-\infty}^{\infty}P_{X_1}(x) + \varepsilon_1\\
&\leq F_{X_1+X_2}(y)+\varepsilon_2 + \varepsilon_1\\
\end{align*}%
Since $E_i(y)$ are non-negative, we also get $F_{X_1'+X_2'}(y) \geq F_{X_1+X_2}(y)$ for all $y$.
\end{proof}

The fact that this trade-off is linear allows us to get a linear approximation error 
in polynomial time, as shown below:


\begin{theorem}
If $X_i' \succeq_{\varepsilon_i} X_i$ for all $i \in\{1,\dots,n\}$ and $\hat{X}= \Sequence(X_1',\dots,X_n', \varepsilon)$ then
$\hat{X} \succeq_e \sum_{i=1}^{n} X_{i}$, where $e={\sum_{i=1}^n\varepsilon_i + n \varepsilon}$. 
\label{appSeqTheorem}
\end{theorem} 

\begin{proof} For $n$ iterations, from Lemma~\ref{Convolv}, we get an accumulated error of 
$\varepsilon_1 +\dots+ \varepsilon_n$. From Lemma~\ref{Trim}, we get an additional error of at most $n\varepsilon$ due to trimming. 
\end{proof}


\begin{theorem} \label{appSeqComplexTheorem}
Assuming that $m \leq 1/\varepsilon$, the procedure
$\Sequence(X_1',\dots,X_n', \varepsilon)$ can be computed
in time $O((nm/\varepsilon)\log(m/\varepsilon))$ using $O(m/\varepsilon)$ memory, where $m$ is the size of the largest support of any of the $X'_i$s.
\end{theorem} 

\begin{proof}
From Lemma~\ref{SizeD}, the size of list $D$ in 
Algorithm~\ref{alg:Trim} is at most $m/\varepsilon$
just after the convolution, after which it is trimmed, so the space complexity 
is $O(m/\varepsilon)$.
$\operatorname{Convolve}$ thus takes time of $O((m/\varepsilon)\log(m/\varepsilon))$, where the logarithmic factor
is required internally
for sorting. Since the runtime of  the $\Trim$ operator is linear, and the
outer loop iterates $n$ times, the overall 
run-time of the algorithm is $O((n m/\varepsilon) \log(m/\varepsilon))$.
\end{proof}



%
%
The following example shows that our error bound is tight, that is,
a sequence of random variables where the error actually achieves
the bounds.

\begin{example}\label{exp:seq}
	Let $0 {\leq} \varepsilon {<}1$ and $n {\in} \mathbb{N}$ such that $1{-}\varepsilon {>} {\varepsilon }/{n}$, i.e.,  
	$\varepsilon$ is small or $n$ is large. 
	Consider, for $\delta>0$ that we will choose to be very small, the random variable $X_1$ defined by
	$$
	P_{X_1}(x) {=} \begin{cases}
	\delta  & \text{if } x=0, \\
	{\varepsilon }/{(n (1-\delta )^{x})} & \text{if } x\in\{1,\dots,n\}, \\
	1{-}\delta{-}\sum_{x=1}^n \frac{\varepsilon }{n(1-\delta )^{x}} & \text{if } x=n+1, \\
	0 & \text{otherwise}
	\end{cases}
	$$
	and, for $i\in\{2,\dots,n\}$, let the random variables $X_i$ be:
	$$
	P_{X_i}(x) =\begin{cases}
	1-\delta  & \text{if } x=0, \\
	\delta     & \text{if } x=n^2, \\
	0 & \text{otherwise}
	\end{cases}
	$$
	The distribution of $X=X_1+X_2$ is
	$$
	P_{X}(x) {=} 
	\begin{cases}
	\delta(1-\delta)  & \text{if } x=0, \\
	{\varepsilon }/{n} & \text{if } x=1, \\
	{\varepsilon }/{(n (1-\delta )^{x-1})} & \text{if } x\in\{2,\dots,n\}, \\
	(1{-}\delta)P(X_1{=}n{+}1)& \text{if } x=n+1, \\
	\delta P(X_1{=}x{-}n^2)& n^2 {\leq} x {\leq} n^2{+}n{+}1 \\
	0 & \text{otherwise}
	\end{cases}
	$$
	The idea here is that the convolution with $X_2$ results in a random variable that is similar in ``shape'' to $X_1$, 
	if we ignore numbers that tend to zero as $\delta$ approaches zero. The convolution also
	modifies the probability $P_{X}1)$ from slightly greater than  $\varepsilon/n$ to precisely $\varepsilon/n$, which will then allow it
	to be trimmed.
	
	Then, if we apply $\Trim(X_1+X_2,\varepsilon/n)$, when $\delta$ is sufficiently small, we get the random variable $X'$ whose probability distribution is:
	$$
	P_{X'}(x) {=} 
	\begin{cases}
	\delta(1-\delta)+{\varepsilon }/{n}  & \text{if } x=0, \\
	{\varepsilon }/{(n (1-\delta )^{x-1})} & \text{if } x\in\{2,\dots,n\}, \\
	1-P(X'{<}n+1)& \text{if } x=n+1, \\
	0 & \text{otherwise.}
	\end{cases}
	$$ 
	
	Note that indeed trimming shifts the mass from $P_X(1)=\varepsilon/n$ to $P_{X'}(0)$.
	This repeats in all steps so, after $n$ steps, we get a random variable $X'$ such that $P_{X'}(0) \xrightarrow{\,\,\delta \to 0\,\,} \varepsilon$. Therefore, 
	$P_{\Sequence(X_1,\dots,X_n,\varepsilon/n)}(0) {-} P_{X_1{+}\dots{+}X_n}(0)$
	approaches $\varepsilon$ as $\delta$ approaches zero
	which means that there exists no $\varepsilon' {<} \varepsilon$
	such that $\Sequence(X_1,\dots,X_n,\varepsilon/n) \succeq_{\varepsilon'} X_1 + \cdots + X_n$ for all $\delta >0$.
	
\end{example}

Observe that if we replace all upper Kolmogorov bound approximations by 
lower Kolmogorov bound approximations, all the results in this section
still hold. Therefore, to obtain lower Kolmogorov bounds all that must be
done is to repeat the computations using $\Trim_L$, that is, keeping
the distribution representation sorted in reverse order.

\section{Parallel nodes}\label{sec:par}

\SetKwFunction{Parallel}{Parallel}

Unlike sequence composition, the deadline problem for parallel composition
is easy to compute, since the execution time of
a parallel composition is the maximum of the durations:
\begin{align}
F_{\max_{i\in[1:n]}X_i}(T)
{=}F_{\bigwedge_{i=1}^n X_i}(T) 
{=}\prod_{i=1}^n F_{X_i} (T)
\label{eq:par}
\end{align}

\noindent where the last equality follows from independence of the random variables.
We denote the construction of the CDF using Equation~\eqref{eq:par} by $\Parallel(X_1,\dots , X_n)$.
If the random variables  are all discrete with finite support, $\Parallel(X_1,\dots , X_n)$
incurs linear space, and computation time $O(nm log(n))$.

If the task tree consists only of parallel nodes, one can
compute the exact CDF, with the same overall runtime.
However, when the task tree contain both sequence and parallel nodes we may get
only approximate CDFs as input, and now the above straightforward computation can compound the errors.
When the input CDFs are themselves approximations, we bound the resulting error:

\begin{lemma} \label{appPalTheorem}
For discrete random variables  $X_1', \dots X_n'$, $X_1, \dots,X_n$, if for all $i=1,\dots,n$,  $X_i' \succeq_{\varepsilon_i} X_i$ and $0\leq\varepsilon_i\leq  \frac{1}{n (K n+1)}$ for some $K>0$,
then, for any $\varepsilon \geq \varepsilon_i$, we have: $\max_{i\in[1:n]}X_i' \succeq_{e} \max_{i\in[1:n]}X_i$ where $e=\sum_{i=1}^n \varepsilon_i + \varepsilon/K$.
\end{lemma}
\begin{proof}
$F_{\max_{i\in[1:n]}X'_i}
 (T) {-} F_{\max_{i\in[1:n]}X_i}
 (T)$
{
\begin{align*} 
&\leq \prod_{i=1}^n (F_{X_i}(T)+\varepsilon_i) {-} \prod_{i=1}^n F_{X_i}(T) \\
&\leq \prod_{i=1}^n (1+\varepsilon_i) - 1 
\leq 1+\sum_{i=1}^n \varepsilon_i + \sum_{k=2}^n {n \choose k}  \varepsilon^k - 1 \\
&\leq \sum_{i=1}^n \varepsilon_i + \!\!\!\!\underbrace{\sum_{k=2}^n n^k  \varepsilon^k}_{\text{sum of a geo. series}}\!\!\!\!
\leq \sum_{i=1}^n \varepsilon_i + \frac{n^2 \varepsilon^2}{1-n \varepsilon}
\leq \sum_{i=1}^n \varepsilon_i + \varepsilon/K\\
\end{align*}}%
Since $F_{X_i'}(T) > F_{X_i}(T)$ for each $i$, this expression is nonnegative.
\end{proof}

Both in Lemma~\ref{appPalTheorem} and in Lemma~\ref{tightPalTheorem} we suggest an upper bound of the error resulted in the case where the input CDFs themselves are approximations. However, Lemma~\ref{appPalTheorem} is designed to facilitate the proof of Theorem~\ref{th:TTalgcomplexity} and Lemma~\ref{appPalTheorem} is designed to facilitate the proof Theorem~\ref{theorem:approxEps0}. 
\begin{lemma} \label{tightPalTheorem}
For discrete random variables $X_1', \dots X_n'$, $X_1, \dots,X_n$, if for all $i=1,\dots,n$,  $X_i' \succeq_{\varepsilon_i} X_i$,
then $\max\{X_1',\dots, X_n'\} \succeq_{e} \max\{X_1,\dots,X_n\}$ where $e=1-\prod_{i=1}^n (1-\varepsilon_i)$.
\end{lemma}
\begin{proof}
When $F_{X'}(T) >0$, we can write:
{
\begin{align*} 
&P(\max\{X_1',\dots, X_n'\} {\leq} T) - P(\max\{X_1,\dots, X_n\} {\leq} T)\\
&=\prod_{i=1}^n F_{X_i'}(T)-\prod_{i=1}^n F_{X_i}(T)\\
&= \prod_{i=2}^n F_{X_i'}(T)(F_{X_1'}(T) - 
F_{X_1}(T) \prod_{i=2}^n \frac{F_{X_i}(T)}{F_{X_i}(T)}) \\
&\leq \prod_{i=2}^n F_{X_i'}(T)(F_{X_1'}(T) - 
   (F_{X_1'}(T) - \varepsilon_1)
   \prod_{i=2}^n \frac{F_{X_i}(T)}{F_{X_i}(T)}) \\
&\leq \prod_{i=2}^n F_{X_i'}(T)(1 -  (1 - \varepsilon_1)
   \prod_{i=2}^n \frac{F_{X_i}(T)}{F_{X_i}(T)}) \\
\end{align*}}%
where the first inequality holds because 
$F_{X_1}(T) \leq F_{X_1'}(T) - \varepsilon_1$, and
the second inequality holds because 
$1 \geq F_{X_1}(T) > 0$ and the second product is positive
and not greater than 1. Since we also have 
that $F_{X_i'}(T) - F_{X_i}(T) \geq 
\varepsilon_i \geq 0$ for every $i$, these steps can be repeated, for $n\geq i \geq 2$, to get the expression: 
$1-\prod_{i=1}^n  (1-\varepsilon_i)$ as claimed. 
Due to monotonicity of max,
we also have $F_{X'}(T) - F_{X}(T) \geq 0$, which
completes the proof.
\end{proof}

\begin{example}\label{expl:parallel}
Let $0 {\leq} \varepsilon {<}1$, 
the random variable $X$ defined by
$$
P_{X}(x) =\begin{cases}
1-\varepsilon  & \text{if } x=0, \\
\varepsilon     & \text{if } x=1, \\
0 & \text{otherwise}
\end{cases}
$$
The "trimmed" version of $X$ in respect to $\varepsilon$, denoted by $X'$, is:
$$
P_{X'}(x) =\begin{cases}
1  & \text{if } x=0, \\
0 & \text{otherwise}
\end{cases}
$$
Let $X_1, \dots, X_n$ be $n$ independent copies of $X_1$ and let $X'_1, \dots, X'_n$ be their ``trimmed versions''. Then:
{
\begin{align*} 
&F_{\max\{X_1',\dots, X_n'\}}( 0) - F_{\max\{X_1,\dots, X_n\}} (0)\\
&=\prod_{i=1}^n (1)-\prod_{i=1}^n (1-\varepsilon)= 1-\prod_{i=1}^n (1-\varepsilon)
\end{align*}}
If we consider now a task tree with a single parallel aggregation level whose children are $n$ 
sequence nodes where the $i$th sequence node has a single primitive-task child modeled by $X_i$, we get that our
computation will introduce exactly the higher bound predicted by Lemma~\ref{tightPalTheorem}, i.e, this bound is tight.
\end{example}

\section{Task trees: mixed sequence/parallel}

Given a task tree $\tau$ and a accuracy requirement $0<\varepsilon<1$, we generate a distribution for a random variable $X_{\tau}'$ approximating the true duration distribution  $X_{\tau}$ for the task tree. 
We introduce the algorithm and prove that the algorithm indeed returns an $\varepsilon$-approximation of the completion time of the plan. 
For a node $v$, let $\tau_v $ be the sub tree with $v$ as root and let $\ch(v)$ be the set of children of $v$.
We use the notation $|\tau|$ to denote the total number of nodes in $\tau$.

\begin{algorithm}
  \DontPrintSemicolon
  \SetKwFunction{Network}{Network} 
  \SetKwFunction{Sequence}{Sequence}
  \SetKwFunction{Parallel}{Parallel}
  
  
  {

 Let $v$ be the root of $\tau$  // Hence, $\tau_v=\tau$ \;
	$n_v = |\ch(v)|$ \;
  \If{$v$ is a $\operatorname{Primitive}$ node}{
	\Return the distribution of $v$\;
			}
	
  \If{$v$  is a $\operatorname{Sequence}$ node}{
  	\For{$c \in \ch(v)$} {
	    $X'_c$ = \Network($\tau_c$, $\frac{|\tau_c|\varepsilon}{|\tau_v|}$)\;
	 }
	 \Return \Sequence($\{X'_c\}_{c \in \ch(v)}$, $\frac{\varepsilon}{n_v|\tau_v|}$)	
	 
	}
			
	\If{$v$ is a $\operatorname{Parallel}$ node}{
	 	\For{$c \in \ch(v)$} {
			$X'_c$ = \Network($\tau_c$, $\min(\frac{|\tau_{c}|\varepsilon }{|\tau_v|}   ,\frac{1}{n_v (|\tau_v| n_v+1)})$)\;
		}
		\Return \Parallel($\{X'_c\}_{c \in \ch(v)}$)
	}	
  }
  
\caption{Network($\tau$, $\varepsilon$)}\label{alg:approx}
  
\end{algorithm}


Algorithm~2, that implements the operator $\Network$, is a straightforward postorder traversal of the task tree. The only remaining issue is handling the error,  in an amortized approach, as seen in the proof of the following theorem.

\begin{theorem}\label{th:TNapprox}
Given a task tree $\tau$, let $X_{\tau}$ be a random variable representing the true
distribution of the completion time for the network. Then $\Network (\tau , \varepsilon)\succeq_\varepsilon X_{\tau}$.
\end{theorem}

\begin{proof} 
By induction on $|\tau|$. 
Base: $|\tau|=1$, the node must be primitive, and $\Network$ will just return the distribution unchanged which is obviously an $\varepsilon$-approximation
of itself. Suppose the claim is true for $1 \leq |\tau| < n$. Let $\tau$ be a task tree of size $n$ and let $v$ be the root of $\tau$. If $v$ is a $\operatorname{Sequence}$ node, by the induction hypothesis that $X'_c\succeq_{|\tau_c|\varepsilon/|\tau_v|} X_{c}$, and by Theorem~\ref{appSeqTheorem}, the maximum accumulated error is 
$\sum_{c \in \ch(v)}
{|\tau_c|\varepsilon}/|\tau_v| + {\varepsilon}/|\tau_v|$ = $(n-1)\varepsilon/|\tau_v|+{\varepsilon}/|\tau_v|= \varepsilon$ for $v$, therefore, $\Sequence(\{X'_c\}_{c\in \ch(v)}, {\varepsilon}/{n}) \succeq_\varepsilon X_\tau$ as required. 
If $v$ is a $\operatorname{Parallel}$ node, by the induction hypothesis that $X'_c  \succeq_{e_c} X_{c}$, where $e_c = \min(\frac{|\tau_{c}|\varepsilon}{|\tau_v|}, \frac{1}{n_v (|\tau_v| n_v+1)})$. 
 
\noindent So $\sum_{c\in \ch(v)} e_c \leq   \sum_{c\in \ch(v)} \frac{|\tau_{c}|\varepsilon}{|\tau_v|}  \leq \varepsilon  - \varepsilon/ |\tau_v|  $. Then, by Lemma~\ref{appPalTheorem}, using $K=|\tau_v|$ and $n=n_v$, we get that $\Parallel(\{X'_c\}_{c\in \ch(v)})\succeq_\varepsilon  X_\tau$ as required. 
\end{proof}

\begin{theorem}\label{th:TTalgcomplexity}
Let  $N$ be the size of the  task tree  $\tau$,  and $M$ the size of the maximal support of each of the primitive tasks.
If $0 \leq \varepsilon \leq \frac{1}{N (N^2+1)}$ and $M < N/\varepsilon$, the  $\Network$ approximation algorithm 
runs in time $O((N^5 /\varepsilon^2)\log(N^3/\varepsilon^2))$, using
$O(N^3/\varepsilon^2) $ memory.
\end{theorem}

\begin{proof}
The run-time and space bounds can be derived from the bounds on $\Sequence$ and on $\Parallel$, as follows. In the $\Network$ algorithm, the trim accuracy parameter is less than or equal to ${\varepsilon}/{N}$. 
The support size (called $m$ in Theorem~\ref{appSeqComplexTheorem}) of the variables input to $\Sequence$ are $O(N^2/\varepsilon)$.
Therefore, the complexity of the $\Sequence$ algorithm is 
$O((N^4 /\varepsilon^2)\log(N^3/\varepsilon^2))$ and the complexity of the $\Parallel$ operator is $O((N^3/\varepsilon)\log(N))$. The time and space for sequence dominate, so the total time complexity is $N$ times the complexity of $\Sequence$ and the space complexity is that of $\Sequence$.
\end{proof}

If the constraining assumptions on $M$ and $\varepsilon$ in Theorem~\ref{th:TTalgcomplexity} are lifted, the complexity is still polynomial:
replace one instance of $1/\varepsilon$ by $max(m,1/\varepsilon)$, and the other by $max(1/\varepsilon,N (N^2+1))$
in the runtime complexity expression.

\section{Tightening the error estimation}\label{Chap:Bound}

Until now, we presented an approximation algorithm (\Network) and provided a bounds for time-accuracy trade-offs. In some cases, however, our algorithm provides results that are much more accurate than promised. This may be wasteful, because this extra precision comes with a price of runtime.

In this section we propose a tighter analysis. The term tight here means that we provide the smallest error estimation possible for a given tree structure. In other words, given a task tree, we provide the smallest error bound that is true for any choice of the leaves, i.e., the random variables.

Consider, as an extreme example, a simple task tree with a sequence node at the root and below it only parallel nodes. The total error of our approximation algorithm (\Network) in this case is only due to the invocation of \Sequence at the tail of the recursion. The \Network algorithm, however,  acts as if all the other nodes add additional errors. Eventually, the \Sequence
computation produces a very small error which takes the toll of an unnecessary computation time. 

We will present now a recursive approximation algorithm with a tighter bound on the error parameter for every sequence node. Specifically, we propose the generalized algorithm $\GenNetwork$ (listed as Algorithm~\ref{alg:GenNetwork} below) whose main property is given using the algorithm $\EstimateError$ (listed as Algorithm~\ref{alg:EstimateError} below) as follows:


\begin{algorithm}
  \DontPrintSemicolon
  \SetKwFunction{Sequence}{Sequence}
  \SetKwFunction{Parallel}{Parallel}
  \SetKwFunction{EstimateError}{EstimateError}
  {

 Let $v$ be the root of $\tau$\;
  \If{$v$ is a $\operatorname{Primitive}$ node}{
	\Return $0$ \;
			}
	
  \If{$v$  is a $\operatorname{Sequence}$ node}{
  	\For{$c \in \ch(v)$} {
	    $e_c=$ \EstimateError($\tau_c$, $\varepsilon$)\;
	 }
	 \Return $\varepsilon(v) + \sum_{c\in \ch(v)}e_{c}$	
	}
	\If{$v$ is a $\operatorname{Parallel}$ node}{
	 	\For{$c \in \ch(v)$} {
	    $e_c=$ \EstimateError($\tau_c$, $\varepsilon$)\;
	 }
		\Return $1-\prod_{c\in \ch(v)} (1-e_{c})$
	}	
  }
  
\caption{EstimateError($\tau$, $\varepsilon$) }
\label{alg:EstimateError}
  
\end{algorithm}


\begin{algorithm}
  \DontPrintSemicolon
  \SetKwFunction{Network}{Network} 
  \SetKwFunction{Sequence}{Sequence}
  \SetKwFunction{Parallel}{Parallel}
  
  
  {

 Let $v$ be the root of $\tau$\;
  \If{$v$ is a $\operatorname{Primitive}$ node}{
	\Return the distribution of $v$\;
			}
	
  \If{$v$  is a $\operatorname{Sequence}$ node}{
  	\For{$c \in \ch(v)$} {
	    $X'_c = \GenNetwork(\tau_c, \varepsilon)$\;
	 }
	 \Return \Sequence($\{X'_c\}_{c \in \ch(v)}$, $\varepsilon(v)$)	
	 
	}
			
	\If{$v$ is a $\operatorname{Parallel}$ node}{
	 	\For{$c \in \ch(v)$} {
			$X'_c = \GenNetwork(\tau_c, \varepsilon)$\;
		}
		\Return \Parallel($\{X'_c\}_{c \in \ch(v)}$)
	}	
  }
  
\caption{ GenNetwork($\tau$, $\varepsilon$)}
\label{alg:GenNetwork}
  
\end{algorithm}

\begin{theorem}\label{theorem:approxEps0}
	For a task tree $\tau$ and a function $\varepsilon$ that maps the sequence nodes in $\tau$ to numbers in $[0,1]$, let 
	$\varepsilon_0=\EstimateError(\tau,\varepsilon)$, and let $X_{\tau}$ be a random variable for the true
	distribution of the completion time of $\tau$; then, $\GenNetwork(\tau, \varepsilon) \succeq_{\varepsilon_0} X_{\tau}$. 
\end{theorem}

\begin{proof}
By induction over the depth of $\tau$, denoted by $d$. Base: $ d=1$, a single primitive node $p$. 
In this case, by line 3 of Algorithm~\ref{alg:GenNetwork}, $\GenNetwork(\tau_p,\varepsilon)=X_p$ and, by line 3 of Algorithm~\ref{alg:EstimateError},  $\EstimateError(\tau_p,\varepsilon)=0$ and the claim follows.
Induction hypothesis: Assume the lemma is true for a task tree with depth $ <d $. Step: let $r$ be the root of the tree whose depth is $d$, i.e., its child subtrees are of depth smaller than $d$.

If $r$ is a sequence node then,	by Theorem~\ref{appSeqTheorem}, sequence nodes produce an error which is the sum of all child nodes errors and an additional $\varepsilon(r)$ caused by the trim operator $$\varepsilon_0={\varepsilon(r)+\sum_{c \in \ch(r)}{\EstimateError(\tau_c,\varepsilon)}}$$ equivalent to \EstimateError line 7. By the induction hypothesis, each of the children sub trees satisfy the lemma and we get that $\GenNetwork(\tau_{r}, \varepsilon) \succeq_{\varepsilon_0} X_\tau$ as required.
	If $r$ is a parallel node then,
	by Lemma~\ref{tightPalTheorem}, $$\varepsilon_0=1-\prod_{c \in \ch(r)}{(1-\EstimateError(\tau_c,\varepsilon))}$$ as in line 11 of \EstimateError. By the induction hypothesis, each of the child subtrees satisfy the lemma and we get that $\GenNetwork(\tau_{r}, \varepsilon) \succeq_{\varepsilon_0} X_\tau$ as required.
\end{proof}

Based on this theorem, we propose the following pseudo-algorithm for computing a tight approximation for the makespan of a task tree, as follows. Use some symbolic mathematical engine, such as Wolfram Mathematica (\url{www.wolfram.com/mathematica/}), to find a function $\varepsilon$ such that $\EstimateError(\tau,\varepsilon)$ is smaller than the approximation that you want to achieve; then, run $\GenNetwork(\tau,\varepsilon)$ with this $\varepsilon$. This, of course, is only a pseudo algorithm because it is just a template, not dictating how to compute $\varepsilon$.
The following lemma establishes, however, that the computation of $\varepsilon$ is feasible, at least in an approximated form, as it involves reversing a polynomial of relatively small degree:

\begin{lemma}\label{lemma:poly}
Given a task tree $\tau$, if we consider $\EstimateError(\tau,\varepsilon)$ as a function of the variables $\{\varepsilon(v)\}_{v\in Seq}$ where $Seq$ is the set of Sequence nodes in $\tau$, we have a polynomial of degree smaller or equal to $|Seq|$.  
\end{lemma}
\begin{proof}
By induction over the depth of $\tau$ denoted by $d$. Base: if $d=1$ we have a tree with a single primitive node $p$. In this case  $\EstimateError(\tau,\varepsilon)$ which is indeed a polynomial degree zero.
Induction step: we refer to the children nodes of the root node $r$, $\ch(r)$, as a set of sub task trees, each of them with depth smaller than $d$. If $r$, is a sequence node 
$$\EstimateError(\tau,\varepsilon)=\sum_{c\in \ch(r)}\EstimateError(\tau_c,\varepsilon) +  \varepsilon.$$ 
If $r$, is a parallel node  $$\EstimateError(\tau,\varepsilon)=1-\prod_{c\in \ch(r)} (1-\EstimateError(\tau_c,\varepsilon)).$$ 
By the induction hypothesis we have that the degree of $\EstimateError(c,\varepsilon)$ is the number of sequence node in $c$, for each $c$, and the above equalities show that the degree is increased by one if and only if the root is a parallel node. 
\end{proof}


%
%

To establish the tightness of the proposed pseudo-algorithm, we give now an example of a task tree that cannot be estimated better than what is possible with a perfect instantiation of our template (i.e., a solver that gives the best $\varepsilon$). the terms ``cannot be estimated better'' in the preceding sentence refer only to approximation schemes that use only the structure of the tree, i.e., that are invariant to the choice of the random variables in the leaves as formalized in the next theorem:

\begin{theorem}\label{theorem:tight}
There is a task tree $\tau$, a bound $T$, and an $\varepsilon$ 
such that $F_{\GenNetwork(\tau, \varepsilon)}(T) - F_{X_\tau}(T) = \EstimateError(\tau, \varepsilon)$ where $X_\tau$ is a random variable representing the completion time of $\tau$. 
\end{theorem}
\begin{proof}
See Example~\ref{exp:seq} and Example~\ref{expl:parallel}.
\end{proof}


The above result establishes that $\GenNetwork$ is tight in the sense that there is no $\varepsilon_0 < \GenNetwork(\tau, \varepsilon)$ 
such that $\GenNetwork(\tau, \varepsilon)$ is always $\succeq_{\varepsilon_0}$ than $X_{\tau}$. 
If we look more closely at on the examples used
to prove the theorem, we can say more about the tightness of our algorithm, as follows.
Since the examples consist of task trees whose root nodes are both of
type sequence and of type parallel, we get that any algorithm that traverses the tree
recursively like we do, cannot do better than $\GenNetwork$ in terms of computing a
random variable that is $\succeq_{\varepsilon}$ than $X_{\tau}$ with a smaller $\varepsilon$. 
More formally, if we restrict the discussions to algorithms where the trimming of the variables in a subtree
depend only on the structure of the subtree (not on siblings or parents or the CDFs of the variables), then 
$\GenNetwork$ applies the maximal possible trimming (assuming that we have an optimal solution to $\EstimateError$). 
This means that $\GenNetwork$ is an improvement over 
$\Network$ (Algorithm~\label{alg:approx}) that satisfies this assumption. 
Practically, the improvement is in allowing more trimming and, by that, saving unnecessary computations. 
Specifically, both $\GenNetwork$ and $\Network$ are guaranteed to give a satisfactory answer but they may sometimes compute 
an approximation that is better than required in the price of taking more run-time. $\GenNetwork$ is better in that it
takes this extra time only in cases where any algorithm that decides how to trim based only on the shape of the subtree 
(and not on the CDFs of the random variables or on other parts of the tree) would.

\section{Complexity results}\label{sec:complexity}

The deadline problem is NP-hard, 
even for a task tree consisting only of primitive tasks and one sequence node, i.e.
{\em linear plans}~\cite{Russell:2003:AIM:773294,simmons2001planning,aktolga2004java}. 

\begin{lemma} \label{SumDiscreteRV}
Let $Y=\{Y_1,\dots,Y_n\}$ be a set of discrete real-valued random variables  specified by probability mass functions 
with finite supports, $T \in \mathbb{Z}$, and $p\in[0,1]$. 
Then, deciding  whether $F_{\sum_{i=0}^{n} Y_{i}}(T)>p $ is NP-Hard.
\end{lemma}

This lemma was first proved in~\cite{mohring2001scheduling} by a reduction from the \textit{Partition} problem~\cite[problem number SP]{Garey:1990:CIG:574848} and also shown in~\cite{cohen2015estimating}, by reduction from the \textit{SubsetSum} problem~\cite[problem number SP13]{Garey:1990:CIG:574848}.
 
\begin{theorem}
Deciding if the probability that a task tree satisfies a deadline $T$ is above a threshold $p$ is NP-hard.
\end{theorem}

\begin{proof} The makespan of task tree consisting of a single sequence node with $n$ leaf nodes is the sum of $n$ random variables (the completion times of the leaves). Therefore, the theorem follows immediately from Lemma~\ref{SumDiscreteRV}. 
\end{proof}

Our next goal is to show that not only that the above decision problem is NP-hard but also to analyze the hardness of computing the exact probability $F_{\sum_{i=0}^{n} Y_{i}}(T)$. To this end, we note that 
if $P_{Y_i}(y)=1/|\support(Y_i)|$ for every $i$ and every $y\in\support(Y_i)$, computing the probability $F_{\sum_{i=0}^{n} Y_{i}}(T)$ is equivalent, up to scaling by $\prod_{i=1}^n |\support(Y_i)|$, to counting the number of assignments to the random variables  such that $\sum_{i=0}^{n} Y_{i} \leq T$. 

\begin{definition} 
	Given the random variable $Y_1,\dots,Y_n$ and a deadline $T$,
	the \#deadline-probability counting problem is to count the number of assignments to the random variables such that $\sum_{i=0}^{n} Y_{i} \leq T$.
\end{definition}

It is easy to see that \#deadline-probability is in \#P because we can check if an assignment satisfies $\sum_{i=0}^{n} Y_{i} \leq T$ in linear time.

We will show that \#deadline-probability is \#P-complete by providing a reduction from \#knapsack. Recall the definition of $\#\operatorname{knapsack}$~\cite{arora2009computational}: We are given $n$ objects, and together with each object $i\in \{1,\dots , n\}$ we have its integer weight $w_i$, and the total weight $W$ our knapsack can hold. Our objective is to find the number of subsets $K \subseteq \{1, \dots , n\}$ such that $\sum_{i \in K} w_i \leq W$. We call the sets satisfying this weight constraint feasible, and denote them as $S$. Thus, the \#knapsack problem is to compute $|S|$.

\begin{theorem} \label{deadlinSP}
\#deadline-probability is \#P-complete.
\end{theorem}

\begin{proof} By reduction from \#knapsack. Given an instance of \#knapsack, create the two-valued random variables $Y_{1},\dots,Y_{n} $ with $P_{Y_i}(w_i)=1/2$ and $P_{Y_i}(0)=1/2$ and choose $T=W$. By construction, $|S|=2^n \cdot F_{\sum_{i=1}^{n} Y_{i}}(T)$ (where $S$ is as explained above). Since, every assignment is chosen with probability $1/2^n$, we get that $|S|$ is the number of assignments such that $\sum_{i=1}^{n} Y_{i} \leq T$. Thus, if we could count the number of such assignments, we could also count the size of $S$. This establishes that the problem is \#P-complete. 
\end{proof}

Finally, we consider the linear utility function, i.e. the problem of computing an expected makespan of a task network.
Note that although for linear plans the {\em deadline problem} is NP-hard, the {\em expectation problem} is trivial because the expectation
of the sum of random variables  $X_i$ is equal to the sum of the expectations of the $X_i$s.
For {\em parallel nodes}, it is easy to compute the CDF and therefore also easy to
compute the expected value.
Despite that, for task networks consisting of {\em  both} sequence nodes and
parallel nodes, these methods cannot be effectively combined, and in fact, we have:
\begin{theorem}
Computing the expected completion time of a task network is NP-hard.
\end{theorem}

\begin{proof}
By reduction from subset sum, defined as:
given a set $S=\{s_{1},\dots,s_{n}\}$ of integers, and integer target value $T$, is there a subset of $S$ whose sum is exactly $T$? 
Given an instance of \textit{SubsetSum}, create
the two-valued random variables  $Y_{1},\dots,Y_{n} $ with $P_{Y_i}(s_i)=1/2$ and $P_{Y_i}(0)=1/2$. 
By construction, there exists a subset of $S$ summing 
to $T$ if and only if $P_{\sum_{i=0}^{n} Y_{i}}(T)>0$.
Construct random variables  (``primitive tasks'') $Y_i$ .
Denote by $X$ the random variable $\sum_{i=1}^{n} Y_{i}$. Construct one parallel node
with two children, one being the a sequence node having the
completion time distribution
defined by $X$, the other being a primitive task that has a completion time
$T_j$ with probability 1. (We will use more than one such case,
which differ only in the value of $T_j$, hence the subscript $j$).
Denote by $M_j$ the random variable that represents the completion time
distribution of the parallel node, using this
construction, with the respective $T_j$. Now consider computing the
expectation of the $M_j$ for the following cases: $T_1 = T+{1}/{2}$
and $T_2 = T+{1}/{4}$.
Thus we have, for $j\in\{ 1, 2\}$, by construction and the definition of
expectation:
{
\begin{eqnarray*}
E[M_i] &=& T_j F_{X}( T_j) + \sum_{x > T_j} x~P_{X}(x)  \\
          &=& T_j F_{X}( T)+ \sum_{x \geq T+1} x~P_{X}(x)
\end{eqnarray*}}
where the second equality follows from the $Y_i$ all being
integer-valued random variables  (and therefore $X$ is also integer valued).
Subtracting these expectations, we have $E[M_1]-E[M_2]=\frac{1}{4}F_{X}( T)$.
Therefore, using the computed expected values, we can compute
$F_{X}( T)$, and thus also $P(X=T)$, in polynomial time.
\end{proof}

To complete the picture, we also state the complexity of the deadline problem  for trees with only parallel nodes:

\begin{lemma} \label{TPar}
	Let $Y=\{Y_1,\dots,Y_n\}$ be a set of independent random variables specified by
	CDF  and let $T \in \mathbb{R}$. Then,  $F_{\max\{Y_1,\dots,Y_n\}}(T)$
	can be computed in polynomial-time.
\end{lemma}
\begin{proof}
	See~\eqref{eq:par}.
\end{proof}

\section{Empirical Evaluation}\label{sec:exp}

We examine our approximation bounds in practice, and compare the results to
exact computation of the CDF and to a simple stochastic sampling scheme. Three types of task trees
are used in this evaluation:
task trees used as execution plans for the ROBIL team entry in the DARPA
robotics challenge (DRC simulation phase, http://in.bgu.ac.il/en/Pages/news/dar\textunderscore pa.aspx), linear plans (seq), and plans for
the Logistics domain (from IPC2 http://ipc.icaps-conference.org/).
The primitive task distributions were uniform
distributions discretized to $M$ values.  
The plans from the DRC are shown in Figures {\ref{fig:drive}}, {\ref{fig:walk}} and {\ref{fig:pickup}}. For every entry of $M$ in Tables~\ref{tab:runtimes1} and ~\ref{tab:runtimes2} each line is the runtime in seconds, and for every entry of $M$ in Tables~\ref{tab:errors1} and ~\ref{tab:errors2} each line presents the estimation error.

\begin{figure}
	\begin{center}
	\includegraphics[width=\textwidth]{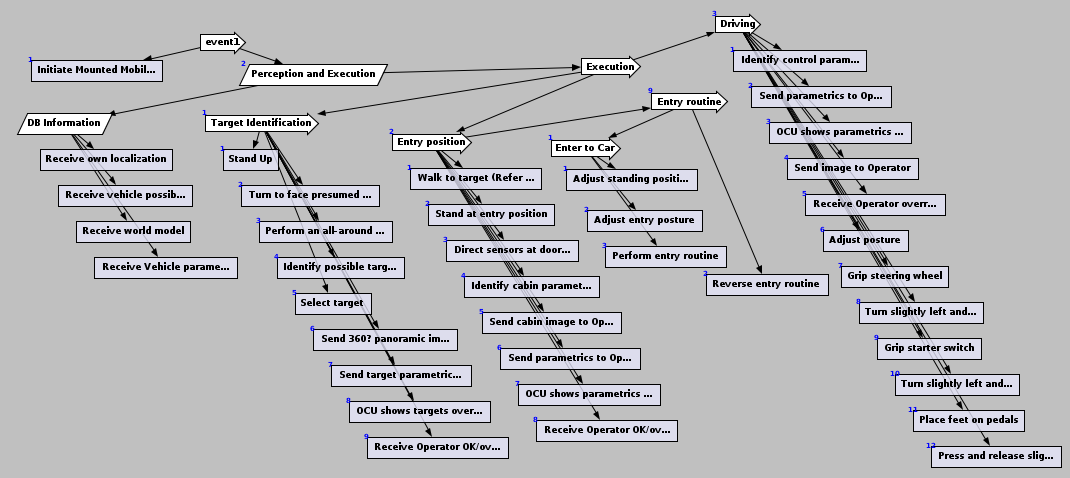}
	\caption{A plan for Drive challenge task, 47 nodes}
	\label{fig:drive}
	\end{center}
\end{figure}

\begin{figure}
	\begin{center}
	\includegraphics[width=\textwidth]{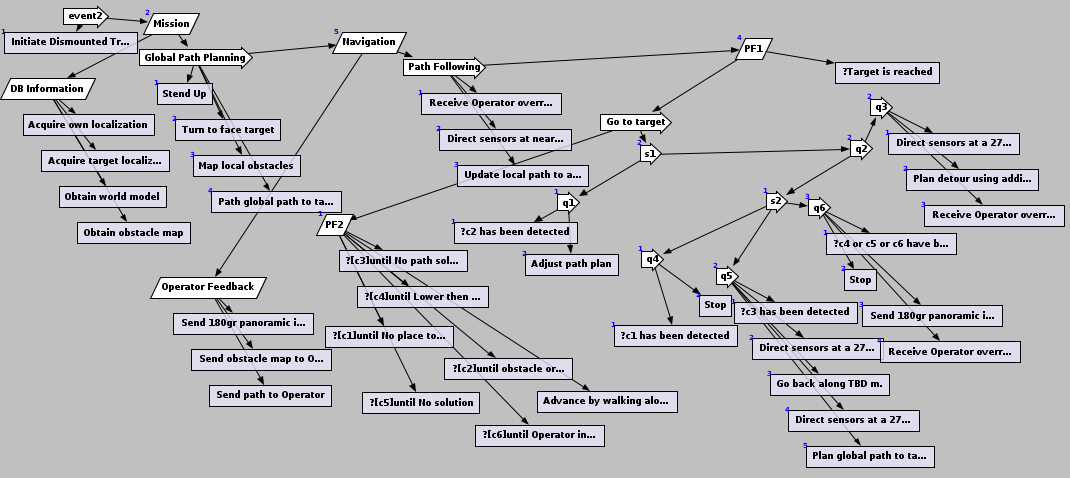}
	\caption{A plan for Walk challenge task, 57 nodes}
	\label{fig:walk}
	\end{center}
\end{figure}

\begin{figure}
	\begin{center}
	\includegraphics[width=\textwidth]{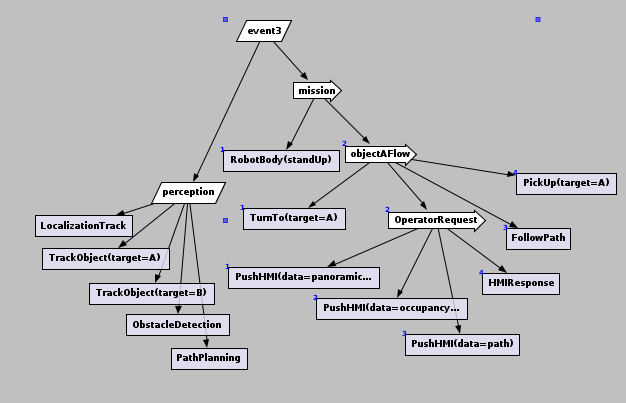}
	\caption{A plan for Pick-Up challenge task}
	\label{fig:pickup}
	\end{center}
\end{figure}

In the Logistics domain, packages are to be transported by trucks or airplanes.
Hierarchical plans were generated by the JSHOP2 planner~\cite{nau2003shop2} for this domain and consisted of one parallel node (packages delivered in parallel),
with children all being sequential plans. In Figure {\ref{fig:transport}} presented a simple plan generated by JSHOP2 algorithm for accomplishing (transport-two p1 p2) from the following initial state: $\{$(package p1), (at p1 l1), (destination p1 l3), (available-truck t1), (at t1 home),
(package p2), (at p2 l2), (destination p2 l3), (available-truck t2), (at t2 home)$\}$.
The duration distribution of all primitive tasks is uniform but the support parameters were determined by the type of the task, 
in some tasks the distribution is fixed (such as for load and unload) and in others the distribution depends on the velocity of the vehicle
and on the distance to be travelled.

After running our approximation algorithm we also ran $\Trim_L$ which uses a reversed version of the $\Trim$ operator, providing a {\em lower} bound of the CDF, as well as the upper bound
generated by Algorithm~\ref{alg:approx}. Running both variants allows us to bound the actual error, costing
only a doubling of the run-time. Despite the fact that our error bound is theoretically tight, in practice
and with actual distributions, according to Tables~\ref{tab:errors1} and ~\ref{tab:errors2}, the resulting error
in the algorithm was usually much better than the theoretical $\varepsilon$ bound.

We ran the exact algorithm, our approximation algorithm with $\varepsilon \in \{ 0.1, 0.01, 0.001\}$, and a simple simulation with 
$10^3$ to $10^7$ samples (number of samples is denoted by $s$ in the table), on networks from the DRC
implementation, sequence nodes with 10, 20, and 50 children (number of nodes denoted by $N$ in the table), and 20 Logistics domain plans, and
several values of $M$ (the notations $M, N$ are as in Theorem~\ref{th:TTalgcomplexity}). 
Results for the various task trees are shown in tables~\ref{tab:errors1},~\ref{tab:errors2} (error comparison) and~\ref{tab:runtimes1},~\ref{tab:runtimes2} (runtime comparison).
Errors are the maximum error in the CDF, measured from the true result when available, and from the bounds generated by the approximation algorithm using $\varepsilon = 0.0001$
when the exact algorithm timed out (over 2 hours). 
The exact algorithm times out in many cases when the number of tasks is 20 or more, except when size of the support $M$ is very small, in which case it handles some more nodes, but still cannot handle 
50 tasks even for $M=2$.
Both our approximation algorithm and the sampling algorithm handle all these cases, as our algorithm's runtime is polynomial in $N$, $M$, and ${1}/{\varepsilon}$
as is the sampling algorithm's (time linear in number of samples). 

The advantage of the approximation algorithm is mainly in providing bounds with
certainty as opposed to the bounds in-probability provided by sampling. 
Additionally, as predicted by theory, accuracy of the approximation algorithm
improves linearly with ${1}/{\varepsilon}$ (and almost linear in runtime), whereas accuracy of sampling improves only as a square root of the number of
samples. Thus, even in cases where sampling initially outperformed the approximation algorithm, increasing the required accuracy for both algorithms,
eventually the approximation algorithm overtook the sampling algorithm.

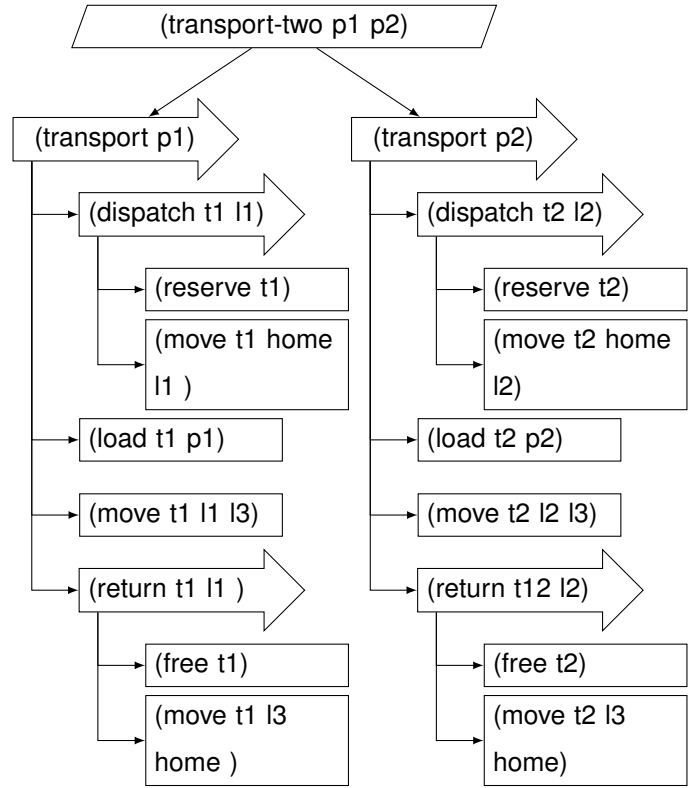
\begin{figure}[htb]
	\centering  
	
		\tikzset{
			basic/.style  = {draw, text width=5cm, font=\sffamily, rectangle},
			root/.style   = {basic, trapezium,trapezium left angle=70,trapezium right angle=-70, thin, align=center},
			level 2/.style = {basic, ,single arrow, thin, align=center,
				text width=7em},
			level 3/.style = {basic, thin, align=left,  text width=7em}
		}
		\begin{tikzpicture}
		 [
		level 1/.style={sibling distance=45mm},
		edge from parent/.style={->,draw},
		>=latex]
		
		\node[root] {(transport-two p1 p2)}
		child {node[level 2 ] (c1) {(transport p1)}}
		child {node[level 2] (c2) {(transport p2)}};
		
		\begin{scope}[every node/.style={level 3}]
		
		\node [single arrow, below of = c1, xshift=25pt] (c11) {(dispatch t1 l1)};
		\node [below of = c11, xshift=25pt] (c111) {(reserve t1)};
		\node [below of = c111] (c112) {(move t1 home l1
			)};
		
		\node [below of = c112, xshift=-25pt] (c12) {(load t1 p1)};
		\node [ below of = c12] (c13) {(move t1 l1 l3)};
		\node [single arrow,below of = c13] (c14) {(return t1 l1
			)};
		\node [below of = c14, xshift=25pt] (c141) {(free t1)};
		\node [below of = c141] (c142) {(move t1 l3 home
			)};

		\node [single arrow,below of = c2, xshift=25pt] (c21) {(dispatch t2 l2)};
		\node [below of = c21, xshift=25pt] (c211) {(reserve t2)};
		\node [below of = c211] (c212) {(move t2 home l2)};
		
		\node [below of = c212, xshift=-25pt] (c22) {(load t2 p2)};
		\node [below of = c22] (c23) {(move t2 l2 l3)};
		\node [single arrow,below of = c23] (c24) {(return t12 l2)};
		\node [below of = c24, xshift=25pt] (c241) {(free t2)};
		\node [below of = c241] (c242) {(move t2 l3 home)};
		
		\end{scope}
		
		\foreach \value in {1,...,4}
		\draw[->] (c1.195) |- (c1\value.west);
		\foreach \value in {1,2}
		\draw[->] (c11.195) |- (c11\value.west);
		\foreach \value in {1,2}
		\draw[->] (c14.195) |- (c14\value.west);

		\foreach \value in {1,...,4}
		\draw[->] (c2.195) |- (c2\value.west); 
		\foreach \value in {1,2}
		\draw[->] (c21.195) |- (c21\value.west);
		\foreach \value in {1,2}
		\draw[->] (c24.195) |- (c24\value.west);  
		
		\end{tikzpicture}

	\caption{A simple plan generated by JSHOP2 algorithm.
		\label{fig:transport}
	}
\end{figure}

\begin{table*}[tbh!]
	{\footnotesize
		\begin{tabular}{|p{0.1cm}|p{0.1cm}|p{0.2cm}|| p{1.05cm}|p{1.05cm}|p{1.4cm}||X{0.6cm}|X{0.6cm}|X{0.6cm}|X{0.6cm}|X{0.6cm}|}
			\hline
			\multicolumn{1}{|c|}{\multirow{2}{*}{Task Tree}} & \multicolumn{1}{c|}{\multirow{2}{*}{$N$}} & \multicolumn{1}{c|}{\multirow{2}{*}{$M$}} & \multicolumn{3}{c|}{Approximation algorithm error, given $\varepsilon$} & \multicolumn{5}{c|}{Sample algorithm error, given \# samples} \\ \cline{4-11} 
			\multicolumn{1}{|c|}{} & \multicolumn{1}{c|}{} & \multicolumn{1}{c|}{} & \multicolumn{1}{c|}{0.1} & \multicolumn{1}{c|}{0.01} & \multicolumn{1}{c|}{0.001} & \multicolumn{1}{c|}{$10^{3}$} & \multicolumn{1}{c|}{$10^{4}$} & \multicolumn{1}{c|}{$10^{5}$} & \multicolumn{1}{c|}{$10^{6}$} & \multicolumn{1}{c|}{$10^{7}$} \\ \hline \hline
			\multirow{3}{*}{Drive} & 47 & 2 & [-0.0052, 0.0086] & [-0.0004, 0.0004] & [-$3.2 {\cdot} 10^{-5}$, $3.4 {\cdot} 10^{-5}$] & 0.0206 & 0.0072 & 0.0031 & 0.0009 & 0.0001 \\ \cline{2-11} 
			& 47 & 4 & [-0.0096, 0.019] & [-0.0009, 0.0013] & [-$9.2 {\cdot} 10^{-5}$, $1.3 {\cdot} 10^{-4}$] & 0.0476 & 0.0075 & 0.0046 & 0.0011 & 0.0001 \\\cline{2-11} 
			& 47 & 10 & [-0.014, 0.028] & [-0.0014, 0.0025] & [-$9.5 {\cdot} 10^{-5}$, $1.4 {\cdot} 10^{-4}$] & 0.0236 & 0.0083 & 0.0024 & 0.0015 & 0.0003 \\ \hline
			\multirow{2}{*}{Walk} & 57 & 2 & [-0.0039, 0.004] & [-0.0003, 0.0003] & [-$3.1 {\cdot} 10^{-5}$, $3.2 {\cdot} 10^{-5}$] & 0.0166 & 0.0067 & 0.002 & 0.0008 & 0.0003 \\ \cline{2-11} 
			& 57 & 4 & [-0.0038, 0.004] & [-0.0004,  0.0004] & [-$3.6 {\cdot} 10^{-5}$, $3.9 {\cdot} 10^{-5}$] & 0.0232 & 0.0125 & 0.0022 & 0.0014 &  0.0003 \\ \cline{2-11} 
			& 57 & 10 & [-0.0047, 0.0049] & [-0.0004, 0.0005] & [-$3.8 {\cdot} 10^{-5}$, $4 {\cdot} 10^{-5}$] & 0.0255 & 0.0117 & 0.0029 & 0.0011  & 0.0003 \\ \hline
			\multirow{2}{*}{Pick Up} & 18 & 10 & {[}-0.0041, 0.0061{]} & {[}-0.0003, 0.0005{]} & {[}-$3.5 {\cdot} 10^{-5}$, $5.8 {\cdot} 10^{-5}${]} & 0.018 & 0.0054 & 0.0027 & 0.0006 & 0.0002 \\ \cline{2-11}
			& 18 & 20 & [-0.0038, 0.0031] & [-0.0006, 0.0005] & [-$3 {\cdot} 10^{-5}$, $3.5 {\cdot} 10^{-5}$] & 0.027 & 0.0046 & 0.0015  & 0.0008 & 0.0002 \\ \hline
			\multirow{2}{*}{Logistics1} & 34 & 2 & [-0.0019, 0.0019] & 0& 0& 0.0168 & 0.007 & 0.001 & 0.0009 & 0.0002 \\ \cline{2-11}  
			& 34 & 4  & [-0.0068, 0.0068]  & [-0.0006, 0.0006] & [-$3.4{\cdot} 10^{-5}$, $3.8{\cdot} 10^{-5}$] & 0.025 & 0.0057 & 0.0032 & 0.0005 & 0.0003\\ \cline{2-11}  
			& 34 & 10  & [-0.008, 0.007] & [-0.0009, 0.0007] & 0 & 0.018 & 0.011 & 0.003 &0.0009 & 0.0004\\ \hline
			\multirow{2}{*}{Logistics2} & 45 & 2 & [-0.002, 0.002] & 0& 0 & 0.013 & 0.015 & 0.004& 0.001 & 0.0003 \\ \cline{2-11}  
			& 45 & 4  & [-0.004, 0.004] & [-0.0004, 0.0004] & [-$3.3{\cdot} 10^{-5}$, $3.4{\cdot} 10^{-5}$]& 0.036 & 0.008 & 0.002 & 0.0006 &0.0002\\ \cline{2-11}  
			& 45 & 10  &[-0.005, 0.006]  & [-0.0004, 0.0006] & 0 & 0.03 & 0.013 & 0.002 &0.001 & 0.0002 \\ \hline
		\end{tabular}
		\caption{Estimation errors}
		\label{tab:errors1}
	}
\end{table*}

\begin{table*}[htb!]
	{\footnotesize
		\begin{tabular}{|X{1.2cm}|l|l|l||l|l|l||X{0.65cm}|X{0.65cm}|X{0.65cm}|X{0.65cm}|X{0.8cm}|}
			\hline
			\multicolumn{1}{|c|}{\multirow{2}{*}{Task Tree}} & \multicolumn{1}{c|}{\multirow{2}{*}{$N$}} & \multicolumn{1}{c|}{\multirow{2}{*}{$M$}} & \multirow{2}{*}{Exact} & \multicolumn{3}{c|}{Approx. algorithm, with $\varepsilon$} & \multicolumn{5}{c|}{Sampling algorithm, with \# samples} \\ \cline{5-12} 
			\multicolumn{1}{|c|}{} & \multicolumn{1}{c|}{} & \multicolumn{1}{c|}{} &  & \multicolumn{1}{c|}{0.1} & \multicolumn{1}{c|}{0.01} & \multicolumn{1}{c|}{0.001} & \multicolumn{1}{c|}{$10^{3}$} & \multicolumn{1}{c|}{$10^{4}$} & \multicolumn{1}{c|}{$10^{5}$} & \multicolumn{1}{c|}{$10^{6}$} & \multicolumn{1}{c|}{$10^{7}$} \\ \hline \hline
			\multirow{2}{*}{Drive} & 47 & 2 & 1.49 & 0.141 & 1.14 & 1.49 & 0.187 & 1.92 & 19.11 & 190.4 & 1905 \\ \cline{2-12} 
			& 47 & 4 & 18.9 & 0.34 & 7.91 & 16.11 & 0.21 & 2.1 & 20.95 &  211.5 & 2113.6\\ \cline{2-12}
			& 47 & 10 & $> 2$h & 1.036 & 32.94 & 390.5 & 0.28 & 2.81 & 28.6 &279.1  &  2844.4  \\ \hline
			\multirow{2}{*}{Walk}  & 57 & 2& 4.46 & 0.33 & 3.1 & 4.03 & 0.205 & 2.06 & 20.86 & 208.1 & 2082.7  \\ \cline{2-12} 
			& 57 & 4  & 183.5 & 0.983 & 18.42 & 95.11 & 0.23 & 2.34 & 23.03 & 230.4 & 2352.4 \\ \cline{2-12} 
			& 57 & 10 & $> 2$h & 8.13 & 128.99 & 3668.2 & 0.293 & 2.92 & 29.16 & 291.3 & 2902.7 \\ \hline
			\multirow{2}{*}{Pick Up} & 18 & 10 & 5.76 & 0.022 & 0.193 & 1.133 & 0.103 & 0.983 & 9.8 & 101.9 & 1006.8 \\ \cline{2-12} 
			& 18 & 20 & 27.88 & 0.046 & 0.4 & 3.15 & 0.132 & 1.33 & 13.25 & 130.4 & 1305.9 \\ \hline
			\multirow{2}{*}{Logistics1} & 34 & 2 & 0.014 & 0.007 &  0.009 & 0.009 & 0.239 & 2.03 & 19.3 & 193.9 & 1767 \\ \cline{2-12} 
			& 34 & 4 & 22.98 & 0.048 & 1.3 & 13.1 & 0.2 &  2 & 20 & 205 &1928\\ \cline{2-12}  
			& 34 & 10 & $> 4$h & 0.25 & 8.26 & 475 & 0.26 & 2.64 & 26.4 & 267 & 2649 \\ \hline 
			\multirow{2}{*}{Logistics2} & 45 & 2 & 0.07 &  0.02 &  0.06 & 0.06 & 0.23 & 2.35 & 23.4 & 234.7 & 2196 \\ \cline{2-12} 
			& 45 & 4 & 373.3& 0.2 & 7 & 82.9 & 0.25 & 2.5 & 25.6 & 256 & 2393\\ \cline{2-12}  
			& 45 & 10 & $> 4$h & 2.19 & 120 & 6101 & 0.31 & 3.12 & 31.3 &314 & 3139\\ \hline
		\end{tabular}
		\caption{Runtime comparison (run times in seconds)}
		\label{tab:runtimes1}
	}
	
\end{table*}

\begin{table*}[tbh!]
	{\footnotesize
		\begin{tabular}{|l|l|l|| p{1.1cm}|p{1.1cm}|p{1.4cm}||p{0.9cm}|p{0.9cm}|p{0.9cm}|}
			\hline
			\multicolumn{1}{|c|}{\multirow{2}{*}{Task Tree}} & \multicolumn{1}{c|}{\multirow{2}{*}{$N$}} & \multicolumn{1}{c|}{\multirow{2}{*}{$M$}} & \multicolumn{3}{c|}{Approximation algo. error, given $\varepsilon$} & \multicolumn{3}{c|}{Sample algo. error, given \# samples} \\ \cline{4-9} 
			\multicolumn{1}{|c|}{} & \multicolumn{1}{c|}{} & \multicolumn{1}{c|}{} & \multicolumn{1}{c|}{0.1} & \multicolumn{1}{c|}{0.01} & \multicolumn{1}{c|}{0.001} & \multicolumn{1}{c|}{$10^{3}$} & \multicolumn{1}{c|}{$10^{4}$} & \multicolumn{1}{c|}{$10^{5}$} \\ \hline \hline
			\multirow{2}{*}{Seq 10} & 10 & 4 & [-0.027, 0.041] & [-0.0027, 0.0041] & [-$2.2 {\cdot} 10^{-4}$, $2.5 {\cdot} 10^{-4}$] & 0.0224 & 0.008 & 0.0017  \\ \cline{2-9} 
			& 10 & 10 &[-0.0316, 0.0615]  & [0.0033, 0.0067] & [$-2.6 {\cdot} 10^{-4}$, $5.2 {\cdot} 10^{-4}$] & 0.027 & 0.0117 & 0.0038  \\ \hline 
			\multirow{2}{*}{Seq 20} & 20 & 2 & [-0.02, 0.0373 ] & [-0.0015, 0.0026] & [-$1.6 {\cdot} 10^{-4}$, $2.6 {\cdot} 10^{-4}$]  & 0.0266 & 0.0077 &  0.003   \\ \cline{2-9} 
			& 20 & 4 & [-0.026, 0.025] & [-0.0025, 0.0025]  & [-$2.7{\cdot} 10^{-4}$, $2.3 {\cdot} 10^{-4}$] & 0.039 & 0.01 & 0.002 \\ \cline{2-9} 
			& 20 & 10   & [-0.027, 0.027] & [-0.0028, 0.0027] & [-$3 {\cdot} 10^{-4}$, $2.5 {\cdot} 10^{-4}$] & 0.032 & 0.007 & 0.0042  \\ \hline
			\multirow{2}{*}{Seq 50} & 50 & 2 & [-0.032, 0.032] & [-0.0028, 0.0028] & [-$2.8 {\cdot} 10^{-4}$, $2.4 {\cdot} 10^{-4}$]  & 0.0193 & 0.007 & 0.0024  \\ \cline{2-9}  
			& 50 & 4  & [-0.035, 0.035]  & [-0.0036, 0.0035] &[-$3.9 {\cdot} 10^{-4}$, $3.2 {\cdot} 10^{-4}$]  & 0.0236 & 0.0064 & 0.0023 \\ \cline{2-9}  
			& 50 & 10  & [-0.037, 0.037] & [-0.004, 0.0039] & [-$4.2 {\cdot} 10^{-4}$, $3.5 {\cdot} 10^{-4}$] & 0.017 & 0.007  & 0.005  \\ \hline
			\multirow{1}{*}{Rand50-AVG} & 50 & 4 & 0.007 & 0.0007 &  0 & 0.0243 & 0.0084 & 0.0024 \\   \hline
			
		\end{tabular}
		\caption{Estimation errors for sequential plans}
		\label{tab:errors2}
	}
\end{table*}

\begin{table*}[htb!]
	{\footnotesize
		\begin{tabular}{|l|l|l|l||l|l|l||p{0.9cm}|p{0.9cm}|p{0.9cm}|}
			\hline
			\multicolumn{1}{|c|}{\multirow{2}{*}{Task Tree}} & \multicolumn{1}{c|}{\multirow{2}{*}{$N$}} & \multicolumn{1}{c|}{\multirow{2}{*}{$M$}} & \multirow{2}{*}{Exact} & \multicolumn{3}{c|}{Approx. algorithm, with $\varepsilon$} & \multicolumn{3}{c|}{Sample algorithm, with \# samples} \\ \cline{5-10} 
			\multicolumn{1}{|c|}{} & \multicolumn{1}{c|}{} & \multicolumn{1}{c|}{} &  & \multicolumn{1}{c|}{0.1} & \multicolumn{1}{c|}{0.01} & \multicolumn{1}{c|}{0.001} & \multicolumn{1}{c|}{$10^{3}$} & \multicolumn{1}{c|}{$10^{4}$} & \multicolumn{1}{c|}{$10^{5}$}  \\ \hline \hline
			\multirow{2}{*}{Seq 10} & 10 & 4 & 0.23 & 0.003 & 0.02 & 0.148  & 0.054 & 0.545 &5.336  \\ \cline{2-10} 
			& 10 & 10 & 10.22 & 0.008 & 0.073 & 0.692 & 0.071 & 0.724 & 7.18  \\ \hline
			\multirow{2}{*}{Seq 20} & 20 & 2 & 0.23 & 0.003 & 0.02 & 0.285  & 0.054 & 0.545 & 9.62  \\ \cline{2-10} 
			& 20 & 4 & $> 2$h & 0.011 & 0.106 & 1.208 & 0.105 & 1.066 & 10.74 \\ \cline{2-10} 
			& 20 & 10 & $> 2$h & 0.035 & 0.331 & 4.67 & 0.145 & 1.473 & 14.38  \\ \hline
			\multirow{2}{*}{Seq 50} & 50 & 2 & $> 2$h & 0.028 &  0.28 & 3.593  & 0.236 & 2.366 & 24.71  \\ \cline{2-10} 
			& 50 & 4 & $> 2$h & 0.079 & 0.81 & 11.145 & 0.265 & 2.68 & 26.84\\ \cline{2-10}  
			& 50 & 10 & $> 2$h & 0.227 & 3.1 & 38.01 & 0.354 & 3.63 & 35.63  \\ \hline
			\multirow{1}{*}{Rand50-AVG} & 50 & 4 & $>2$h & 1.1544 & 19.77 & 390.58 & 5.676 & 55.021 & 590.17 \\   \hline
		\end{tabular}
		\caption{Runtime comparison (run times in seconds) for sequential plans}
		\label{tab:runtimes2}
	}
\end{table*}


\section{Dependencies and other generalizations}\label{sec:generalizations}

Computing the distribution of the makespan in trees is considered a trivial problem in some contexts
in probabilistic reasoning~\cite{Pearl}.
Specifically, given the task network, such as the one in Figure~\ref{fig:task-network},
it is straightforward to represent the 
distribution using a Bayes network (BN) that has one node
per task where the {\em children} of a node $v$ in the 
task network are represented by BN nodes that are {\em parents} of
the BN node representing $v$. This results in a tree-shaped BN, 
where it is well known that probabilistic reasoning can be done in time linear
in the number of nodes, e.g., by belief propagation (message passing)~\cite{Pearl,Kim}. 
However, there is a difficulty, usually ignored in the UAI literature, in the potentially exponential size variable domains, which our algorithm, essentially a limited form
of approximate belief propagation from primitive task variables to the root, avoids by trimming.

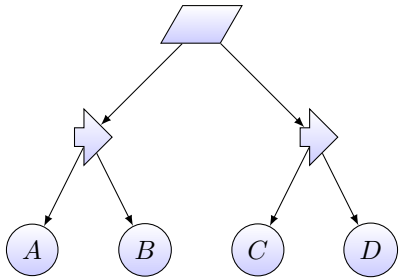
\begin{figure}
\centering
\begin{tikzpicture}[
    level/.style={sibling distance=30mm/#1},
    edge from parent/.style = {draw, -latex},
    sloped,
    every node/.style = {trapezium, trapezium left angle=60, trapezium right angle=-60, draw, align=center, top color=white, bottom color=blue!20,minimum height=0.5cm}
]
   \node {}  
   child { node [single arrow] {} 
        child{ node [circle] {$A$}}
        child{ node [circle] {$B$}}
    }    
    child  { node [single arrow] {} 
        child{ node [circle] {$C$}}
        child{ node [circle] {$D$}}
    }
    ;
\end{tikzpicture}
\caption{A simple task network.}
\label{fig:task-network}
\end{figure}

Looking at makespan distribution computation as probabilistic reasoning leads immediately to 
the question on how to handle task completion times that have dependencies, represented as a BN. 
Since reasoning in BNs is NP-hard even for binary-valued variables~\cite{Dagum.aij,Cooper.ai}, 
this is hard in general.
But for cases where the BN toplogy is tractable, such as for BNs 
with a small cutset,
BNs with bounded treewidth~\cite{Bodlaender:2006:TCA:2092758.2092759},
or directed-path singly connected BNs~\cite{ShimonyDomshlak.aiRN2003},
a deterministic polynomial-time approximation scheme for 
the makespan distribution may be achievable.

\begin{figure}
\centering
\begin{tikzpicture}[
    level/.style={sibling distance=20mm/#1},
    edge from parent/.style = {draw, -latex},
    sloped,
    grow=right,
    every node/.style = {circle, draw, align=center, top color=white, bottom color=blue!20}
]
   \node (S) {}  
        child{ node (C)  {$C$}
        	child{ node (D)  {$D$}  }
        }
        child{ node (A)  {$A$}
        	child{ node (B)  {$B$} }
        }

    ;
    
    \node (T) [right=4cm] {};
    
    \path [draw, -latex] (D) -- (T);
    \path [draw, -latex] (B) -- (T);
    \path [draw, -latex, dashed] (A) -- (D);
\end{tikzpicture}
\caption{A petri-net graph.}
\label{fig:petri-net-graph}
\end{figure}
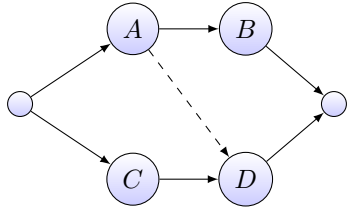

\begin{figure}
\centering
\begin{tikzpicture}[
    level/.style={sibling distance=30mm/#1},
    edge from parent/.style = {draw, -latex},
    sloped,
    every node/.style = {trapezium, trapezium left angle=60, trapezium right angle=-60, draw, align=center, top color=white, bottom color=blue!20,minimum height=0.5cm}
]
   \node {}  
   child { node [single arrow] {} 
        child{ node (A) [circle] {$A$}}
        child{ node [circle] {$B$}}
    }    
    child  { node [single arrow] {} 
        child{  node (X2) {}         
        	child {node [circle] {$C$}}
        }
        child{ node [circle] {$D$}}
    }
    ;
    
    \path [draw, -latex] (X2) [bend left] edge (A);
\end{tikzpicture}
\caption{A non-hierarchical task network}
\label{fig:non-tree-TN}
\end{figure}
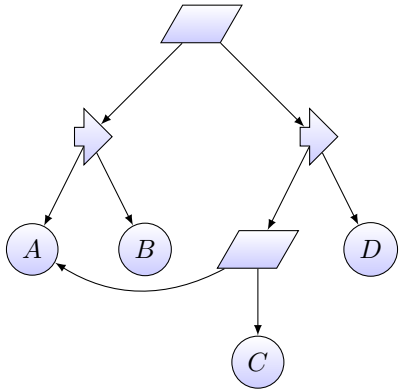

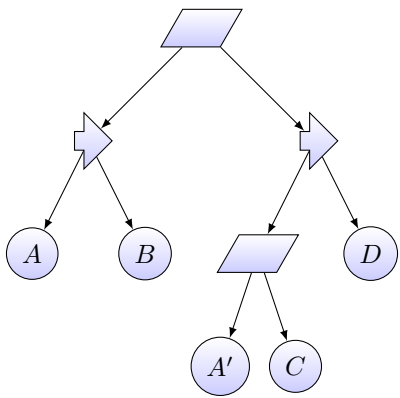
\begin{figure}
\centering
\begin{tikzpicture}[
    level/.style={sibling distance=30mm/#1},
    edge from parent/.style = {draw, -latex},
    sloped,
    every node/.style = {trapezium, trapezium left angle=60, trapezium right angle=-60, draw, align=center, top color=white, bottom color=blue!20,minimum height=0.5cm}
]
   \node {}  
   child { node [single arrow] {} 
        child{ node (A) [circle] {$A$}}
        child{ node [circle] {$B$}}
    }    
    child  { node [single arrow] {} 
        child{  node (X2) {}         
             child {node [circle] {$A'$}}
        	child {node [circle] {$C$}}
        }
        child{ node [circle] {$D$}}
    }
    ;

\end{tikzpicture}
\caption{Representing ``shared'' tasks  using correlated random variables}
\label{fig:task-network-with-copy}
\end{figure}

Here we motivate and handle a special case of small cutsets. 
Specifically, suppose that, in addition to the tree, we allow a small number of dependencies
between primitive task distributions. Does our algorithm generalize to this case?
The importance of this question is because such an extension is natural in some contexts. For example,
in the logistics case we have a 2-level tree, with a toplogy similar to that of Figure ~\ref{fig:task-network}.
The children of the sequence nodes are primitive tasks such as
``drive delivery truck 1 from Boston to NY'' (suppose this is primitive task A in  Figure ~\ref{fig:task-network}).
Now, the duration of this action is a random variable
that depends on the state of traffic at the time the action is taken. Suppose that another
primitive action (e.g. primitive task C) is ``drive delivery truck 2 from Boston to NY'' which is to occur roughly at the
same time as the first action. Since traffic conditions are likely to be very similar, the
duration of the actions may be correlated, and we need to be able to take
this dependency into account.

Another case where we have dependency is when the same primitive action is 
used in more than one composite task. Although this state cannot be represented in
strict hierarchies, recall that timing relationships represented by HTNs can also be
represented by directed acyclic perti nets. For example, the task network of Figure~\ref{fig:task-network}
can be represented by the perti net of Figure~\ref{fig:petri-net-graph} (without the shaded arc). However, perti nets allow more general
timing constraints: the language of trees is equlivalent to perti nets with a series-parallel
graph structure. Adding the shaded arc from A to D in the petri net of Figure~\ref{fig:petri-net-graph}, 
we get a graph that is not series-parallel. Its equivalent in HTNs would be the non-tree structure
shown in Figure~\ref{fig:non-tree-TN}, that shares primitive task A between composite tasks. 
The latter could be converted into a pure tree-shape by adding a task A' that mirrors task A,
i.e. has a duration exactly equal to that of A (Figure~\ref{fig:task-network-with-copy}).

This case, as well as generalizations thereof where the number of correlated variables is small,
we can handle by a scheme known as {\em conditioning}, adapted to our approximation scheme.
For example, in cutset conditioning, a separate reasoning problem is generated
for every possible value instantiation over all the cutset variables. The results
are combined by weighted averaging. We propose to do the same in our case, but must prove
that the approximation quality is maintained, as we indeed do below.

We thus assume that all primitive task durations are independent, when conditioned on
a small cutset $Y$ of the primitive task durations. The joint duration distribution over $Y$
can be provided by a BN, or a complete table, or any other representation. We assume
that the cardinality of the set $Y$ is sufficiently small that the joint domain
size $m^{|Y|}$ is managable, in terms of memory and computation time if we have to
iterate over all domain values. We are also given the duration distribution for any
other task, given every possible value assignment $y$ to the variables in $Y$.
Together, this information fully defines the joint probability of all the 
primitive task durations. 

For example, in the case of Figure~\ref{fig:task-network-with-copy}, we can set $Y=\{ A\}$, so $Y$ is a singleton set.
This is a somewhat degenerate example case, as
the joint distribution can be represented trivially using $P(A'=a'|A=a)=1$,
as  all other primitive task duration variables are independent of $A$.

Let $X$ be the random variable denoting the makespan distribution of the
root of the task network. We wish to estimate $F_X$, the cumulative distribution of $X$.
Our approximation algorithms for trees without dependency can estimate an
upper and lower bounds approximations. With dependencies, we cannot do so directly.
However, consider an assignment $Y=y$, for some value $y\in D(Y)$. We are given the conditional
distribution $Z|Y=y$ for all the rest of the primitive tasks, which are now
independent given $Y=y$. Consider the distribution:
\[
F_{X|Y=y}(x) = P(X\leq x| Y=y)
\]
For each value $Y=y$ we can run the approximation algorithm, to get upper 
Kolmogorov bound $F_{X^+|Y=y}(x)$ and  lower Kolmogorov bound $F_{X^-|Y=y}(x)$. 
Due to Theorem~\ref{th:TNapprox}, we have
the following property, for all $x$:
\[
F_{X^+|Y=y}(x) \succeq_\varepsilon F_{X|Y=y}(x) \succeq_\varepsilon F_{X^-|Y=y}(x)
\]
Now let:
\[
F_{X^-}(x) = \sum_{y\in D(Y)} P_Y(y) F_{X^-|Y=y}(x)
\]
and likewise:
\[
F_{X^+}(x) = \sum_{y\in D(Y)} P_Y(y) F_{X^+|Y=y}(x)
\]

\begin{theorem}
Computing $F_{X^-}(x)$ and $F_{X^+}(x)$ takes time $O(m^{|Y|})t$, where $t$ is the runtime
of our tree task network algorithm. The resulting approximation obeys, for all $x$:
\[
F_{X^+}(x) \succeq_\varepsilon F_{X}(x) \succeq_\varepsilon F_{X^-}(x)
\]
\end{theorem}

\begin{proof}
The runtime bound is obvious, as a trivial implementation simply takes $m^{|Y|}$
runs of the tree task network algorithm. The approximation bounds follow
from the bounds for individual values $Y=y$, and from a convexity argument.
For example, by construction, we have:
\[
F_{X^-}(x) -  F_{X}(x) =  \sum_{y\in D(Y)} P_{Y}(y)(F_{X^-|Y=y}(x) - F_{X|Y=y}(x))
\]
The right hand side is a convex sum of quantities that are all between 0 and $\varepsilon$, which therefore
must also be between 0 and $\varepsilon$.
\end{proof}

\section{Discussion}\label{sec:discussion}

We proposed an operator for trimming the support of random variables such that the resulting trimmed variable is an approximation of the variable that has a bigger support. As the motivation in this paper was to estimate the probabilities of meeting deadlines in hierarchical plans, the notion of approximation used is a one-sided version of the Kolmogorov metric, that reflects the fact that in such estimations we allow over-, not under- approximations. The core of the paper is devoted to an analysis of the prorogation of the estimation errors in the computation of the random variable that represents the makespan of a hierarchical plan. Based on this analysis, the paper proposes recursive algorithms that can compute an approximation of this makespan in time and memory that are polynomial in the sizes of the supports of the primitive tasks, the size of the tree, and of the inverse of the required accuracy ($1/\varepsilon$).

In the following paragraphs we discuss directions for future research and ideas for possible technical improvements of the proposed techniques. Some of these improvements are easy to implement and the reason for not including them in the first place was for clarity of the presentation, other require future research.

\paragraph{Avoid trimming variables with a small support} In the proposed algorithm, for ease of analysis and because we wanted to keep the code simple, we trimmed all the input and intermediate variables, whatever the size of their support is. This may be required, in a worst case, so doing so does not affect the complexity results, but it may give inferior run time and memory performance in the average case. We therefore, recommend to only trim variables that are small. This can be added as an initial test inside the Trim procedure.

\paragraph{Add a trim after a parallel node} Another point is that in the combined algorithm,  space and time complexity can be reduced by adding some $\Trim$ operations,
especially after processing a parallel node, which is not done in our version. This may reduce accuracy,  a trade-off yet to be examined.

\paragraph{Focused trimming} Another option is, when given a specific threshold, trying for higher accuracy in just the region of the threshold, but how to do that is non-trivial. For {\em sampling} schemes such methods are known, including adaptive sampling~\cite{bucher1988adaptive,lipton1990practical}, stratified sampling, and other schemes. It may be possible to apply such schemes to deterministic algorithms as well - an interesting issue for future work. 

\paragraph{Extension to continuous distributions} Our algorithm can handle them by
pre-running a version of the $\Trim$ operator on the primitive task distribution. Since one cannot iterate over support values
in a continuous distribution, start with the smallest support value (even if it is $- \infty$), and find the value at which the CDF
increases by $\varepsilon$. This requires access to the inverse of the CDF, which is available, either exactly or approximately,
for many types of distributions. 

\paragraph{Approximating expectations} We showed that the expectation problem is also NP-hard. A natural question is on approximation algorithms for the expectation problem, but the answer here is not so obvious. Sampling algorithms may run into trouble if the target distribution contains major outliers, i.e. values very far from other values but with extremely low probability. Our approximation algorithm can also be used as-is to estimate the CDF and then to approximate the expectation, but we do not expect it to perform well because our current $\Trim$ operator only limits the amount of probability mass moved at each location to $\varepsilon$, but does not limit the ``distance'' alnog the $x$ parameter over which it is moved. The latter may be arbitrarily bad for estimating the expectation. Nevertheless, a different version of $\Trim$ that bounds just this distance was shown to provide a polynomial-time approximation scheme for the expectations in EXPECTI-MIN-MAX game trees~\cite{shperbergmonte},  if the utilities are bounded. Since the expectation operator involves convolution, these results should be applicable (with some adjustment) to task networks as well.

\paragraph{Optimal trimming} While we proved that the trimming procedure proposed in this paper allows for approximation that improves polynomially with the time and memory invested. It is interesting to look for optimal approximations. In~\cite{cohen2018}, we showed that an optimal approximation of a single random variable can be obtained in polynomial time. Specifically, we showed that given a random variable $X$ and a target support size $m$, we can find the minimal $\varepsilon^*$ and a variable $X'$ such that $X'$ has support of size $m$ and $X \prec_{\varepsilon^*} X'$. Note that this does not directly give an optimal approximation of the makespan of a complete plan.

\paragraph{Compact representations of the random variables} One can view the work presented in this paper in the context of function approximation. In general, a function approximation problem is about the selection a function among a well-defined class that approximates a target function in a certain way. In our case, we approximate the CDF of a random variable with a piecewise constant function with a small number of pieces. As in other applications of function approximation, it is natural to ask whether more compact representations of the random variable exist. For example one can represent functions in a compressed from where a repeated entry can be specified once with a number that specifies the number of repetitions. Another approach would be to approximate using, e.g., splines instead of constant lines. The challenge will be, in any of these variants, to work directly on the compressed representation, as we do in this paper.

\paragraph{Split weights} In the proposed $\Trim$ algorithm, the inner loop goes until $p+prob(d) \leq \varepsilon$ and, when this condition is not met, the value of $prod(d)$ is left for the next iteration. A possible improvement, not included in the base version for simplicity, is to add to $p$ the part of $prob(d)$ up to $\varepsilon$ (i.e., have $p=\varepsilon$) and leave only the remaining part of $prob(d)$ to the next iteration.

\section{Related work}

We outline previous work on HTN planning, series-parallel networks, scheduling with uncertain task durations, the sum and the maximum of random variables, and approximation schemes.

\paragraph{HTN (Hierarchical Task Network)} Some task network models include, beyond the nodes that we handled in this paper, constraints on the tasks that restrict how some of the variables can be bound and the order in which parallel tasks are to be performed~\cite{erol1994htn,erol1996complexity, Russell:2003:AIM:773294}. In~\cite{erol1994htn} Erol et al., formally define, analyze and explicate features of the design of HTN planning systems. Specifically, how is the complexity of HTN planning varies with various conditions on the task networks. Our construction, at moment, supports only the basic structure. Methods for solving HTN are suggested as an online planning~\cite{nau1998control, nau2003shop2, gabaldon2002programming} and as offline planning~\cite{kelly2008offline}. In the experiments we conducted, we used hierarchical plans obtained by SHOP2~\cite{nau2003shop2} planner in the ``Logistics" domain from IPC2 (\href{url}{http://ipc.icaps-conference.org/}). The SHOP2 (Simple Hierarchical Ordered Planner 2) is a domain-independent planning system based on Hierarchical Task Network (HTN) planning. There are other HTN planners like TLPLan or TALPlanner ~\cite{kvarnstrom2000talplanner} but we chose, for convenience, to use JSHOP2, the Java version of SHOP2.

\paragraph{Series-parallel networks} There has been much work on series-parallel networks, although not all related to planning or AI. In~\cite{gelenbe1989multiprocessor}, Gelenbe discusses the fundamental issues involved in the performance of parallel computers. We believe that our work can be applied also in this context. Specifically, in Chapter 5 of this book Gelenbe proposes a model for series-parallel processing structures. Programs in this model are composed of (primitive) tasks; some of them are to be performed in series, others may be performed in parallel. Given the execution time distribution of each task (assuming i.i.d) and the characteristic parameters of the branching process, a method for computing numerically the execution time distribution of the program is shown. The computation involves numerical solutions based on solving a differential-integral equation and iterative methods. In~\cite{gutjahr1992average}, which is based on the same model, a bound on the average total execution time of a series–parallel processing structure is presented. Both papers are very relevant to our work and contributed as case studies (not reported directly in this paper). Moreover, this type of work supplies another motivation for the work shown in this paper. Temporal planing and in particular TPNs (temporal plan network) are presented in~\cite{kim2001executing}, the model is similar to ours, but the focus is on lower/upper bounds, rather than probability distributions. Hierarchical constraint-based plans in MAPGEN~\cite{MAPGEN} allow for more
general dependencies than series-parallel, providing additional expressive power but making the deadline problem even harder.

\paragraph{ Scheduling under uncertainty} Scheduling and in particular, scheduling under uncertainty, can provide additional motivation to our work. In~\cite{herroelen2005project} Herroelen and Leus review approaches for scheduling under uncertainty such as reactive scheduling and stochastic project scheduling and discuss the potentials of these approaches for scheduling under uncertainty of projects (tasks) with deterministic network evolution structure. Another relevant paper is~\cite{hagstrom1988computational} which provides computational complexity results for two PERT problems. Here a project is specified by precedence relations among tasks and task durations specified as discrete independent random variables. Three results are obtained: computing a value of the cumulative distribution function of project duration is \#P-complete, computing the mean of the distribution is at least as hard, and neither of the problems can be
computed in time polynomial in the number of points in the range of the project duration. This paper deals with a more general problem than ours, and with a different type of complexity. Our results are orthogonal, because we show a source of complexity that is not in the graph structure but in the distributions themselves. In fact, the 2-state problem shown to be hard for general graphs by Hagstrom, can be solved in polynomial time for series-parallel trees by dynamic programming. Another relevant paper is~\cite{beck2007proactive} which allows to represent each activity by an independent random variable with a known mean and variance. The best solutions are ones which have a high probability of achieving a good makespan, and methods for combining Monte Carlo simulation with deterministic scheduling algorithms are shown. Compared to our results, the bounds given in~\cite{beck2007proactive} are all in terms of the probability of errors while we bound the errors absolutely. In~\cite{fu2010towards} RCPSP/max (Resource Constrained Project Scheduling Problems with minimum and maximum time lags) is studied, where the durations of activities are not known with certainty. Similarly, in~\cite{bonfietti2014disregarding} a simulation approach is used to evaluate the expected makespan of a number of Partial Order Schedules (POS). The evaluation in the paper shows correlation between the expected makespan and the makespan obtained by simply fixing all durations to their average. The authors of this paper claim that the correlation that they report on allows to use averages instead of the random variables and yet obtain a very accurate estimation of the expected makespan. This, of course, is due to the linearity of expectation, that does not carry when working with other operators such as maximum. Another disadvantage of using averages or sampling is that they cannot give formal guarantees needed, e.g, in Service-Level-Agreements (SLA) where guarantees of the form: ``response time is less than 1mSec in at least 95\% of the cases'' are common~\cite{buyya2011sla}.
A few aspects distinguish our work from the presented scheduling papers. First, both~\cite{fu2010towards} and~\cite{bonfietti2014disregarding} provide guarantees only in probability. These are good for application where such guarantees are sufficient while we provide stronger guarantees. Second, our work is on approximating CDF, i.e., the probability of missing deadlines. Clearly, unlike expectations, this cannot always be directly derived from averages and variances as in~\cite{bonfietti2014disregarding} or by sampling as in~\cite{beck2007proactive}. 
For example, consider the following two tasks. Task A whose duration is 10 seconds with probability 0.999 or 20 seconds with probability 0.001. And task B whose duration is 10.02 seconds with probability 0.999 or 0.02 seconds with probability 0.001. Clearly, these two random variables have the same expectation (10.01) and even the same variance (0.01998). Of course, the averages of samples of these variables should be close to their expectations. But, if we take, say, ten tasks of A in sequence, we get a probability close to zero of crossing a deadline of 100.1 and if we take ten tasks of B in sequence, we get that the probability to cross this deadline is almost 1.

\paragraph{Sum and Max of random variables} Various works exist in the research literature regarding the sum and the maximum of random variables. For example, Evans and Leemis~\cite{evans2004algorithms} present algorithms for computing the exact probability density function of the sum of two independent discrete random variables and show an implementation of the algorithm in a computer algebra system. This paper does not examine the case of summing more than two discrete random variables which is our main challenge.  However, there are other papers that handle the case of summing more than two random variables like~\cite{hong2013computing} which examines the sum of Bernoulli distributed random variables. They present a simple derivation for an exact formula with a closed-form expression for the CDF of the Poisson binomial distribution. The derivation uses the discrete Fourier transform of the characteristic function of the distribution. Numerical studies were conducted to study the accuracy of the developed algorithm and approximation methods.
For example Bromiley~\cite{bromiley2013products} examines the sum of normally distributed random variables. The fact that the product and the convolution of Gaussian probability density functions (PDFs) are also Gaussian functions, is well known. Bromiley provides proofs that the following cases are also Gaussian functions: the product of two univariate Gaussian PDFs, the product of an arbitrary number of univariate Gaussian PDFs, the product of an arbitrary number of multivariate Gaussian PDFs, and the convolution of two univariate Gaussian PDFs. Note that this list does not include the maximum of two variables. In fact, we are not aware of any representation of random variables, except for the implicit PMF table used in this paper, that is closed under addition and under maximum. Mercier~\cite{mercier2007discrete} proposes algorithms for computing bounds of cumulative density functions of sums of i.i.d. non-negative random variables, renewal functions and cumulative density functions of geometric sums of i.i.d. non-negative random variables. Our work allows for non identical variables. Distribution of the maximum of random variables are discussed in~\cite{devroye1980generating}, with a focus mostly on continuous distributions. All the above papers treat either the sum of random variables or the maximum of random variables but not both together. 

\paragraph{Approximation algorithms} Our work also relates to FPTAS (fully polynomial-time approximation schemes)~\cite{arora2009computational} approximation algorithms for the, so called, knapsack problem~\cite{lai2006knapsack, lawler1979fast,gopalan2011fptas}. The idea of  FPTAS for knapsack is to scale the profits downwards enough so the profits of all the objects are polynomially bounded in $n$ and then to use dynamic programming on the new instance. By scaling with respect to a desired $\varepsilon$, the solution is at least $(1-\varepsilon)\cdot OPT$ where $OPT$ is the optimal solution, in polynomial time with respect to both $n$ and $1/\varepsilon$. Our binning technique is similar. Another algorithm that uses a similar binning technique, for a variant of the, so called, subsetsum problem, is described in~\cite{Cormen:2001:IA:580470}. This subsetsum variant is a decision problem, given a set of $n$ numbers $x_1, \cdots x_n$, a target $t$ and $0\leq\varepsilon\leq 1$, return yes, if there is a subset that sums between $(1 - \varepsilon)t$ and $t$. In this context, FPTAS ``trimming" is to remove values that are close to each other. In other words, if two values $s_1$ and $s_2$ that represent a sum of subset of $S$ are close to each other, then for the purpose of finding an approximate solution there is no reason to maintain both of them, so it is possible to merge them to be represented by $s_1$ or $s_2$ and delete the other. The idea in our approximation algorithm is similar to this. In~\cite{arora2009computational}, Chapter 9  ``complexity of counting", the complexity class \#P is presented and examples for counting problems are given, e.g., \#knapsack problem (defined in Chapter \ref{sec:complexity}). Another type of approximation is FPRAS, studied e.g., in~\cite{dyer2003approximate} which presents an algorithm that uses dynamic programming to provide a deterministic relative approximation and then sampling techniques to give arbitrary approximation ratios  and ~\cite{morris2004random}  that uses Markov chain Monte Carlo technique. The research literature also contains numerous {\em randomized} approximation schemes that handle dependencies~\cite{Pearl,Yuan20061189}, especially for
the case with {\em no evidence}. In fact, our implementation of the sampling scheme in ROBIL handled dependent durations.
It is unclear whether such sampling schemes can be adapted to handle dependencies {\em and} arbitrary evidence, such as: ``the completion time of
compound task $X$ in the network is known to be exactly one hour from now''. Another type of approximation algorithms uses Monte-Carlo technique as presented, e.g., in~\cite{bucher1988adaptive,lipton1990practical}. 

\paragraph{Acknowledgments.} This research was supported by the ROBIL project, and by the Lynne and William Frankel Center for Computer Science at Ben-Gurion University.

\section*{Refferences}
\bibliography{ijcai15}

\begin{thebibliography}{10}
\expandafter\ifx\csname url\endcsname\relax
  \def\url#1{\texttt{#1}}\fi
\expandafter\ifx\csname urlprefix\endcsname\relax\def\urlprefix{URL }\fi
\expandafter\ifx\csname href\endcsname\relax
  \def\href#1#2{#2} \def\path#1{#1}\fi

\bibitem{cohen2015estimating}
L.~Cohen, S.~E. Shimony, G.~Weiss, Estimating the probability of meeting a
  deadline in hierarchical plans, in: Twenty-Fourth International Joint
  Conference on Artificial Intelligence, 2015.

\bibitem{gabaldon2002programming}
A.~Gabaldon, Programming hierarchical task networks in the situation calculus,
  in: AIPS’02 Workshop on On-line Planning and Scheduling, 2002.

\bibitem{erol1994htn}
K.~Erol, J.~Hendler, D.~S. Nau, {HTN} planning: Complexity and expressivity,
  in: AAAI, 1994.

\bibitem{nau1998control}
D.~S. Nau, S.~J. Smith, K.~Erol, et~al., Control strategies in {HTN} planning:
  Theory versus practice, in: AAAI/IAAI, 1998, pp. 1127--1133.

\bibitem{nau2003shop2}
D.~S. Nau, T.-C. Au, O.~Ilghami, U.~Kuter, J.~W. Murdock, D.~Wu, F.~Yaman,
  {SHOP2}: An {HTN} planning system, J. Artif. Intell. Res. (JAIR) 20 (2003)
  379--404.

\bibitem{kelly2008offline}
J.~P. Kelly, A.~Botea, S.~Koenig, Offline planning with {H}ierarchical {T}ask
  {N}etworks in video games., in: AIIDE, 2008, pp. 60--65.

\bibitem{bonfietti2014disregarding}
A.~Bonfietti, M.~Lombardi, M.~Milano, Disregarding duration uncertainty in
  partial order schedules? {Yes}, we can!, in: CPAIOR, 2014, pp. 210--225.

\bibitem{buyya2011sla}
R.~Buyya, S.~K. Garg, R.~N. Calheiros, {SLA}-oriented resource provisioning for
  cloud computing: {Challenges}, architecture, and solutions, in: Cloud and
  Service Computing (CSC), 2011.

\bibitem{mohring2001scheduling}
R.~M{\"o}hring, Scheduling under uncertainty: Bounding the makespan
  distribution, Computational Discrete Mathematics (2001) 79--97.

\bibitem{lilliefors1967kolmogorov}
H.~W. Lilliefors, On the {K}olmogorov-{S}mirnov test for normality with mean
  and variance unknown, Journal of the American Statistical Association
  62~(318) (1967) 399--402.

\bibitem{Russell:2003:AIM:773294}
S.~J. Russell, P.~Norvig, Artificial Intelligence: A Modern Approach, 2nd
  Edition, Pearson Education, 2003.

\bibitem{simmons2001planning}
R.~Simmons, {Planning, Execution \& Learning 1. Linear \& Non-Linear Planning},
  (Script). Carnegie Mellon University, USA.

\bibitem{aktolga2004java}
E.~Aktolga, A java planner for blocksworld problems, University of Osnabrueck.

\bibitem{Garey:1990:CIG:574848}
M.~R. Garey, D.~S. Johnson, Computers and Intractability; {A} Guide to the
  Theory of {NP}-Completeness, W. H. Freeman \& Co., NY, USA, 1990.

\bibitem{arora2009computational}
S.~Arora, B.~Barak, Computational complexity: a modern approach, Cambridge
  University Press, 2009.

\bibitem{Pearl}
J.~Pearl, Probabilistic Reasoning in Intelligent Systems: {Networks} of
  Plausible Inference, Morgan Kaufmann, San Mateo, CA, 1988.

\bibitem{Kim}
J.~H. Kim, J.~Pearl, A computation model for causal and diagnostic reasoning in
  inference systems, in: IJCAI, 1983.

\bibitem{Dagum.aij}
P.~Dagum, M.~Luby, Approximating probabilistic inference in {B}ayesian belief
  networks is {NP}-hard, Artificial Intelligence 60 (1) (1993) 141--153.

\bibitem{Cooper.ai}
G.~F. Cooper, The computational complexity of probabilistic inference using
  {B}ayesian belief networks, Artificial Intelligence 42 (2-3) (1990) 393--405.

\bibitem{Bodlaender:2006:TCA:2092758.2092759}
H.~L. Bodlaender, Treewidth: Characterizations, applications, and computations,
  in: WG, 2006, pp. 1--14.

\bibitem{ShimonyDomshlak.aiRN2003}
S.~E. Shimony, C.~Domshlak, Complexity of probabilistic reasoning in
  directed-path singly connected {B}ayes networks, Artificial Intelligence 151
  (2003) 213--225.

\bibitem{bucher1988adaptive}
C.~G. Bucher, Adaptive sampling{:} an iterative fast {M}onte {C}arlo procedure,
  Structural Safety 5~(2) (1988) 119--126.

\bibitem{lipton1990practical}
R.~J. Lipton, J.~F. Naughton, D.~A. Schneider, Practical selectivity estimation
  through adaptive sampling, Vol.~19, ACM, 1990.

\bibitem{shperbergmonte}
S.~S. Shperberg, S.~E. Shimony, A.~Felner, Monte-carlo tree search using batch
  value of perfect information.

\bibitem{cohen2018}
L.~Cohen, T.~Grinshpoun, G.~Weiss, Optimal approximation of random variables
  for estimating the probability of meeting a plan deadline, in: AAAI, 2018.

\bibitem{erol1996complexity}
K.~Erol, J.~Hendler, D.~S. Nau, Complexity results for {HTN} planning, Annals
  of Mathematics and Artificial Intelligence 18~(1) (1996) 69--93.

\bibitem{kvarnstrom2000talplanner}
J.~Kvarnstr{\"o}m, P.~Doherty, {TALplanner}: {A} temporal logic based forward
  chaining planner, Annals of Mathematics and Artificial Intelligence 30~(1-4)
  (2000) 119--169.

\bibitem{gelenbe1989multiprocessor}
E.~Gelenbe, Multiprocessor performance, Wiley, 1989.

\bibitem{gutjahr1992average}
W.~J. Gutjahr, G.~C. Pflug, Average execution times of series--parallel
  networks, S{\'e}minaire Lotharingien de Combinatoire 29 (1992) 9.

\bibitem{kim2001executing}
P.~Kim, B.~C. Williams, M.~Abramson, Executing reactive, model-based programs
  through graph-based temporal planning, in: IJCAI, 2001, pp. 487--493.

\bibitem{MAPGEN}
M.~Ai-Chang, J.~Bresina, L.~Charest, A.~Chase, J.~Cheng-jung Hsu, A.~Jonsson,
  B.~Kanefsky, P.~Morris, K.~Rajan, J.~Yglesias, B.~G.~Chafin, W.~C.~Dias,
  P.~Maldague, Mapgen: Mixed-initiative planning and scheduling for the mars
  exploration rover mission 19 (2004) 8--12.

\bibitem{herroelen2005project}
W.~Herroelen, R.~Leus, Project scheduling under uncertainty: {Survey} and
  research potentials, European journal of operational research 165~(2) (2005)
  289--306.

\bibitem{hagstrom1988computational}
J.~N. Hagstrom, Computational complexity of {PERT} problems, Networks 18~(2)
  (1988) 139--147.

\bibitem{beck2007proactive}
J.~C. Beck, N.~Wilson, Proactive algorithms for job shop scheduling with
  probabilistic durations., J. Artif. Intell. Res.(JAIR) 28 (2007) 183--232.

\bibitem{fu2010towards}
N.~Fu, P.~Varakantham, H.~C. Lau, Towards finding robust execution strategies
  for {RCPSP/max} with durational uncertainty., in: ICAPS, 2010, pp. 73--80.

\bibitem{evans2004algorithms}
D.~L. Evans, L.~M. Leemis, Algorithms for computing the distributions of sums
  of discrete random variables, Mathematical and Computer Modelling 40~(13)
  (2004) 1429--1452.

\bibitem{hong2013computing}
Y.~Hong, On computing the distribution function for the poisson binomial
  distribution, Computational Statistics \& Data Analysis 59 (2013) 41--51.

\bibitem{bromiley2013products}
P.~A. Bromiley, Products and convolutions of gaussian probability density
  functions (2013).

\bibitem{mercier2007discrete}
S.~Mercier, Discrete random bounds for general random variables and
  applications to reliability, European j. of operational research 177~(1)
  (2007) 378--405.

\bibitem{devroye1980generating}
L.~Devroye, Generating the maximum of independent identically distributed
  random variables, Computers \& Mathematics with Applications 6~(3) (1980)
  305--315.

\bibitem{lai2006knapsack}
K.~Lai, M.~Goemans, The knapsack problem and fully polynomial time
  approximation schemes ({FPTAS}), Retrieved November 3 (2006) 2012.

\bibitem{lawler1979fast}
E.~L. Lawler, Fast approximation algorithms for knapsack problems, Mathematics
  of Operations Research 4~(4) (1979) 339--356.

\bibitem{gopalan2011fptas}
P.~Gopalan, A.~Klivans, R.~Meka, D.~Stefankovic, S.~Vempala, E.~Vigoda, An
  {FPTAS} for\# knapsack and related counting problems, in: Foundations of
  Computer Science (FOCS), 2011 IEEE 52nd Annual Symposium on, IEEE, 2011, pp.
  817--826.

\bibitem{Cormen:2001:IA:580470}
T.~H. Cormen, C.~Stein, R.~L. Rivest, C.~E. Leiserson, Introduction to
  Algorithms, 2nd Edition, McGraw-Hill Higher Education, 2001.

\bibitem{dyer2003approximate}
M.~Dyer, Approximate counting by dynamic programming, in: Proceedings of the
  thirty-fifth annual ACM symposium on Theory of computing, ACM, 2003, pp.
  693--699.

\bibitem{morris2004random}
B.~Morris, A.~Sinclair, Random walks on truncated cubes and sampling 0-1
  knapsack solutions, SIAM journal on computing 34~(1) (2004) 195--226.

\bibitem{Yuan20061189}
C.~Yuan, M.~J. Druzdzel,
  \href{http://www.sciencedirect.com/science/article/pii/S0895717705005443}{Importance
  sampling algorithms for {Bayesian} networks: {Principle}s and performance},
  Mathematical and Computer Modelling 43~(9–10) (2006) 1189 -- 1207.
\newblock \href {http://dx.doi.org/http://dx.doi.org/10.1016/j.mcm.2005.05.020}
  {\path{doi:http://dx.doi.org/10.1016/j.mcm.2005.05.020}}.
\newline\urlprefix\url{http://www.sciencedirect.com/science/article/pii/S0895717705005443}

\end{thebibliography}

\end{document}